%% file: main.tex
\documentclass{article}

\PassOptionsToPackage{numbers, compress}{natbib}

\usepackage[preprint]{neurips_2020}

\usepackage[utf8]{inputenc} 
\usepackage[T1]{fontenc}    
\usepackage{hyperref}       
\usepackage{url}            
\usepackage{booktabs}       
\usepackage{amsfonts}       
\usepackage{nicefrac}       
\usepackage{microtype}
\usepackage{graphicx}
\graphicspath{{figures/}}
\usepackage{subcaption}
\usepackage{booktabs} 
\usepackage{pgfplotstable}
\usepackage{ifthen}
\usepackage{colortbl}
\usepackage[nosumlimits,nonamelimits]{amsmath}
\usepackage{amssymb}
\usepackage{mathtools}
\usepackage{multirow}
\usepackage{tabularx}
\usepackage{todonotes}
\usepackage{float}
\usepackage{algorithm}
\usepackage{algorithmic}
\usepackage[inline]{enumitem}
\usepackage[english]{babel}
\usepackage{wrapfig}

\usepackage{amsthm}

\DeclareMathSizes{10}{9}{7}{5}
\theoremstyle{definition}
\newtheorem{definition}{Definition}

\newtheorem{theorem}{Theorem}
\newtheorem{lemma}{Lemma}
\newtheorem{assumption}{Assumption}

\usepackage{titlesec}
\titlespacing*{\section}{0pt}{0.1\baselineskip}{0.1\baselineskip}
\titlespacing*{\subsection}{0pt}{0.1\baselineskip}{0.1\baselineskip}
\titlespacing*{\subsubsection}{0pt}{0.1\baselineskip}{0.1\baselineskip}

\title{Sign-regularized Multi-task Learning}

\author{%
Johnny Torres\textsuperscript{1} \thanks{\textsuperscript{1} Escuela Superior Politecnica del Litoral (ESPOL); jomatorr@espol.edu.ec \textsuperscript{2}George Mason University (GMU)}
\And
Guangji Bai\textsuperscript{2}
\And
Junxiang Wang\textsuperscript{2} 
\And 
Liang Zhao\textsuperscript{2}
\And
Carmen Vaca\textsuperscript{1}
\And
Cristina Abad\textsuperscript{1}
}

\begin{document}

\maketitle

\begin{abstract}
Multi-task learning is a framework that enforces different learning tasks to share their knowledge to improve their generalization performance. It is a hot and active domain that strives to handle several core issues; particularly, which tasks are correlated and similar, and how to share the knowledge among correlated tasks. Existing works usually do not distinguish the polarity and magnitude of feature weights and commonly rely on linear correlation, due to three major technical challenges in: 1) optimizing the models that regularize feature weight polarity, 2) deciding whether to regularize sign or magnitude, 3) identifying which tasks should share their sign and/or magnitude patterns. To address them, this paper proposes a new multi-task learning framework that can regularize feature weight signs across tasks. We innovatively formulate it as a biconvex inequality constrained optimization with slacks and propose a new efficient algorithm for the optimization with theoretical guarantees on generalization performance and convergence. Extensive experiments on multiple datasets demonstrate the proposed methods' effectiveness, efficiency, and reasonableness of the regularized feature weighted patterns.
\end{abstract}

\section{Introduction}

In the real world, many learning tasks are correlated and have shared knowledge and patterns. For example, the sentiment analysis models built for the texts in different domains (e.g., sports, movie, and politics) exhibit 
shared patterns (e.g., emotion, icons) and exclusive patterns (e.g., domain-specific terminologies). Multi-task learning is a machine learning framework which makes it possible to
different learning tasks to share their common knowledge yet preserve their exclusive characteristics and eventually improve the generalization performance of all the tasks. Multi-task learning has been applied into many types of learning tasks such as supervised, unsupervised, semi-supervised, and reinforcement learning tasks as well as numerous applications such as recommendation systems~\cite{10.1145/3298689.3346997}, natural language understanding~\cite{liu-etal-2019-multi}, computer vision~\cite{Liu_2019_CVPR}, and self-driving cars~\cite{liang2019multi}.

Multi-task learning is an active domain that has attracted much attention and research efforts. The key challenge in multi-task learning is how to selectively transfer information among the related tasks while preventing sharing knowledge between unrelated tasks, also known as \emph{negative transfer}~\cite{pan2009survey}. To achieve this, we must identify: 1) which tasks are correlated, and 2) which types of knowledge can be shared among the correlated tasks. Many research efforts have been recently devoted to address these two core issues. While most of the methods assume that all the tasks are correlated, fast-increase amount of methods are devoted to automatically identify which tasks are correlated and which are not, usually with some assumptions on the correlation patterns such as  tree-structured ~\cite{gornitz2011hierarchical,wang2018incomplete}, clustered~\cite{zhou2011clustered, zhou2013learning, liu2017adaptive}, and graph-structured~\cite{jacob2009clustered, zhou2011clustered, kumar2012learning, zhou2015flexible, liu2017adaptive, yao2019robust}. The price is the increase of computational complexity, risk of over-fitting, and difficult of optimization~\cite{bishop2006pattern} (e.g., involving discrete optimization). To attack the second core issue, different types of shared knowledge have been proposed in terms of model parameters (i.e., feature weights), by assuming that different tasks should share similar typically in terms of magnitude and linear correlation, such as \emph{magnitude of feature weights} (e.g., via $\ell_{1,2}$ norm~\cite{liu2009multi, jalali2010dirty, wang2016multiplicative}), latent topics (e.g., via enforcing \emph{low-rank structure of feature weights} \cite{evgeniou2004regularized, argyriou2008convex, kumar2012learning}), and \emph{value of feature weights} (e.g., squared loss among the feature weights). Existing works typically focus on regularizing the magnitude or the similarity among the weights instead of merely their signs. The readers can refer to~\cite{zhang2017survey} for a comprehensive survey.

Despite the large amount of existing works, many types of real-world tasks correlations cannot be extensively covered, especially those without tight and linear correlations. Frequently, it is appropriate to merely enforce similar polarity but not magnitude of feature weights across different tasks. For example, we may assume the term ``happy'' to contribute positively to the sentiment of a text in both of tasks of movie rating and presidential election, but we do not further require that their strength of importance be similar. In other cases, a feature $A$ that is more important than feature $B$ in one task might not necessarily indicate that it should be more important than $B$ in another task. However, the above issues have not been well explored due to several technical challenges including: \textbf{1). Difficulties in optimizing the models that regularize feature weight polarity.} Feature weight signs involve discrete functions, which makes it difficult to jointly optimize with those continuous optimization problems in current multitask learning frameworks. \textbf{2) Incapability in deciding whether to regularize sign or magnitude.} It is difficult for existing methods to automatically learn and distinguish when the features' weights should share same signs and when to further share similar importance across tasks. \textbf{3) Challenges in identifying which tasks should share their sign and/or magnitude patterns.} Not all tasks may satisfy the regularization on sign and magnitude of weights. We need an efficient algorithm with theoretical guarantees for identifying task relations that satisfy sign regularization.

To address the above challenges, we propose a new Sign-Regularized Multi-task Learning (\textbf{SRML}) framework that adaptively regularizes weight signs across tasks. The contributions of the paper include:
\vspace{-0.2cm}
\begin{itemize}
    \setlength\itemsep{-0.2em}
    \item \textbf{A new robust multi-task learning framework.} Our framework is able to regularize different tasks to share the same weight signs. It can automatically identify which tasks and features can share their weight signs.
    \item \textbf{A new algorithm for parameter optimization.} The learning model has been innovatively formulated as a biconvex inequality constrained problem with slacks. New efficient optimization algorithm has been proposed based on nonconvex alternating optimization.
    \item \textbf{Theoretical properties and guarantee of the algorithm has been analyzed.} Theoretical merits of the proposed algorithm including convergence, convergence rate, generalization error, and time complexity have been analyzed.
    \item \textbf{Extensive experiments have been conducted.} We have demonstrated the model effectiveness and efficiency in 
    5 real-world datasets and 
    3 synthetic datasets, under the comparison with other multi-task learning frameworks. Further analyses on the learned feature weight patterns reveal the effectiveness of the proposed regularization on weight signs.
\end{itemize}

\section{Sign-regularized Multi-task Learning (SRML)}

This section introduces the SRML problem which encourages same weight polarity across multiple tasks during multi-task learning.
 

Define a multi-task learning problem with $T$ tasks, where 
the set of tasks $t \in \{1, \dots ,T\}$ 
associated with a set of instances, $X_t\in \mathbb{R}^{m_t \times d}$ represent the input data while $y_t \in \mathbb{R}^{m_t}$ is the target variable. Here $m_t$ denotes the number of instances for task $t$. 

\begin{definition}[Sign-regularized Multi-task Learning] For all the tasks $\{1,\cdots,T\}$, our goal is to learn $T$ predictive mappings, where for each task $t$ the mapping is $f:\mathbb{R}^{m_t \times d}|w_t\rightarrow\mathbb{R}^{1 \times m_t}$ where the mapping function $f$ is parameterized by $w_t\in\mathbb{R}^{d\times 1}$, where $\forall i\ne j:\mathrm{sign}(w_i)=\mathrm{sign}(w_j)$.
\label{def:srml}
\end{definition}

Unlike most of the multi-task learning frameworks that regularize the magnitude of the weights, the problem defined in Definition \ref{def:srml} is a new type of multi-task problem that regularizes over the signs of the weights. Such assumption is usually weaker and easier to satisfy in many types of applications, where tasks only need to share their knowledge on whether each feature should contribute positively or negatively to the prediction. The learning objective for the SRML problem is formulated as follows:
\vspace{-0.2cm}
\begin{equation}
\begin{split}
  \min_{w_1,\cdots,w_T} & \displaystyle\sum_{t=1}^T\mathcal{L}_t(w_t) + \lambda\Omega(\{w_t\}^\intercal) \ \ \ \ \text{s.t.,  }  { w_{t,j}w_{t+1,j}\geq 0 } \ \ \forall {t=1,\cdots,T-1, j=1,\cdots,d}
\end{split}
\label{eq:wmtl}
\end{equation}
where the inequality constraint enforces the same signs of each feature $j$ across different tasks. $\mathcal{L}_t(w_t) = L(f(w_t,X_t),Y_t)$ where $L(\cdot,\cdot)$ denotes commonly-used loss function such as squared loss for regression and logistic loss for classification ~\cite{bishop2006pattern}. $\Omega(\{w_t\}_t^\intercal)$ is additional regularization over all the parameters if $\lambda\ne 0$.
Equation \eqref{eq:wmtl} assumes all the tasks must completely share their polarity of weights, which may be too strict considering the possible noise and negative transfer among tasks. To enhance robustness and in the meanwhile allow automatic identification of those tasks who cannot share their weight signs, we add relaxations to it, leading to the following:
\vspace{-0.1cm}
\begin{equation}
\begin{split}
& \min_{w_1,\cdots,w_T,\xi} \displaystyle\sum_{t=1}^T\mathcal{L}_t(w_t) + \lambda\displaystyle\sum_{t=1}^T\Omega_t(w_t) + c\displaystyle\sum_{t=1}^{T-1}\xi_{t} \\
& \text{s.t.} \ \ w_{t,j} w_{t+1,j} + \xi_{t,j} \geq 0,\ \ \xi_{t,j} \geq 0\text{,} \ \ t=1,2,\dots,T-1, j=1,2,\dots,d
\end{split}
\label{eq:obj_swmtl}
\end{equation}
where each $\xi_t \in \mathbb{R}^d$ is a slack variable, and $c$ is a hyperparameter for controlling the level of slacking.

\section{Optimization Method}

The optimization objective in \autoref{eq:obj_swmtl} is nonconvex with biconvex terms in inequality constraints. Moreover, there will be a huge number of constraints when the numbers of tasks and features are huge. There is no existing efficient methods that can handle this challenging new problem with theoretical guarantee. Although efficiency-enhanced methods such as Alternating Direction Methods of Multipliers (ADMM)~\cite{boyd2011distributed,wang2017nonconvex} are demonstrated to accelerate classical Lagrangian methods, extending them to handle nonconvex-inequality constraints are highly nontrivial. To address such issues, this paper proposes a new efficient algorithm based on ADMM for handling such nonconvex-inequality-constrained problem. Theoretical analyses on convergence properties, generalization error, and complexity analysis are also provided.

By adding auxiliary variable $u$, the problem in Equation \eqref{eq:obj_swmtl} can be transformed into Equation \eqref{eq:obj_swmtl_3}:
\vspace{-0.2cm}
\begin{equation}
\begin{split}
 \min_{w_t, u_t}  \displaystyle\sum_{t=1}^T[\mathcal{L}_t(w_t)& + \lambda\Omega(\{w_t\}_t^\intercal)] +  
 c\displaystyle\sum_{t=1}^{T-1}\displaystyle\sum_{j=1}^d \max {(0,-u_{t,j}u_{t+1,j})} \\
  \text{s.t.\quad} &  {[w_1;\cdots;w_T]-[u_1;\cdots;u_T]=0}
\end{split}
\label{eq:obj_swmtl_3}
\end{equation}
where we can see the original problem has been transformed into a new problem with much simpler constraints. Moreover, the smooth part, nonsmooth part, nonconvex part, and constraints have been separated and easy for being formed into separate subproblems each of which is much easier to solve efficiently. The augmented Lagrangian form is as follows:
\vspace{-0.1cm}
\begin{equation}
\begin{split}
 L_{\rho} & = \displaystyle\sum_{t=1}^T[\mathcal{L}_t(w_t) +  \lambda\left\lVert w_{t} \right\rVert_{2}^{2}] + c\displaystyle\sum_{t=1}^{T-1}\displaystyle\sum_{j=1}^d \max {(0,-u_{t,j}u_{t+1,j})} \\
& + y^\intercal([w_1;\cdots;w_T]-[u_1;\cdots;u_T]) + (\rho/2)\left\lVert [w_1;\cdots;w_T]-[u_1;\cdots;u_T]\right\rVert_{2}^{2}
\end{split}
\label{eq:aug_lagrangian_swmtl_3}
\end{equation}
where $y$ is the dual variable. $\rho$ is the penalty parameter that controls the trade-off between primal and dual residual during optimization. Then we can perform alternating optimization upon Equation \eqref{eq:aug_lagrangian_swmtl_3} for alternately optimizing all the variables until convergence, which is detailed in Section \ref{optimization method} in supplementary material.

\section{Theoretical Analysis}

In this section, we will present the theoretical properties of our SRML model and algorithms.

\subsection{Generalization Error Bound}

We first provide an equivalent transformation on our original problem and then give the generalization error bound for it. Our original slacked weakly multi-task learning problem with $L_1$-norm regularization can be written as:
\vspace{-0.1cm}
\begin{equation}
\begin{split}
 \min_{w,\xi} \ \ & \frac{1}{T}\displaystyle\sum_{t=1}^T\frac{1}{m}\sum_{i=1}^{m}\mathcal{L}(\langle w_t,x_{ti}\rangle,y_{ti}) 
 + \lambda\displaystyle\sum_{t=1}^T\left\lVert w_t \right\rVert_{1} + c\displaystyle\sum_{t=1}^{T-1}\left\lVert\xi_t \right\rVert_{1} \\
& \text{s.t.} \ \ w_{t} \otimes w_{t+1} + \xi_{t} \geq 0, \xi_{t} \geq 0 \text{; } \ \  t=1,2,\dots,T-1
\end{split}
\label{eq:original_swmtl_model}
\end{equation}
where the $\otimes$ operator for any two vector $a,b\in\mathbb{R}^n$  is defined as:
$a\otimes b\coloneqq (a_{1}b_{1},a_{2}b_{2},\cdots,a_{n}b_{n})^{T} \label{eq:ele-wise-product}$. 
By combining the constraints with respect to $\xi$, we can simplify the constraints and get:
\vspace{-0.1cm}
\begin{equation}
 \min_{w} \frac{1}{T}\displaystyle\sum_{t=1}^T\frac{1}{m}\sum_{i=1}^{m}\mathcal{L}(\langle w_t,x_{ti}\rangle,y_{ti}) + \lambda\displaystyle\sum_{t=1}^T\left\lVert w_t \right\rVert_{1} 
 + c\cdot\displaystyle\sum_{t=1}^{T-1}\left\lVert max(\vec{\scriptstyle 0},-w_{t} \otimes w_{t+1}) \right\rVert_{1}
\label{eq:transformed_swmtl_model_2}
\end{equation}
Here the $max$ function is operated element-wisely.
We can prove by simply using the Lagrangian that \autoref{eq:transformed_swmtl_model_2} could be equivalently transformed into the following one with a new set of parameters:
\vspace{-0.1cm}
\begin{equation}
\begin{split}
 \min_{w} \ \  & \frac{1}{T}\displaystyle\sum_{t=1}^T\frac{1}{m}\sum_{i=1}^{m}\mathcal{L}(\langle w_t,x_{ti}\rangle,y_{ti}) \ \ \ \ \text{s.t.} \ \ 
\displaystyle\sum_{t=1}^T\left\lVert w_t \right\rVert_{1} \leq \alpha 
\text{,} \ \  \displaystyle\sum_{t=1}^{T-1}\left\lVert max(\vec{\scriptstyle 0},-w_{t} \otimes w_{t+1}) \right\rVert_{1}\leq\beta
\end{split}
\label{eq:transformed_swmtl_model_1}
\end{equation}

\begin{assumption}
\label{ass:obj_function}
The loss function $\mathcal{L}$ in this paper has values in $[0,1]$ and has Lipschitz constant $L$ in the first argument for any value of the second argument, i.e.: 
\begin{enumerate*}
    \item $\mathcal{L}(\langle w,x\rangle,y) \in [0,1]$ 
    \item $\mathcal{L}(\langle w,x\rangle,y) \leq L\langle w_{t},x\rangle, \forall{y}$.
\end{enumerate*}
\end{assumption}

\begin{definition}
(Expected risk, Empirical risk). 
Given any weights $w$, we denote the expected risk as:
\vspace{-0.1cm}
\begin{equation}
\begin{split}
\mathbb{E}(w) \coloneqq \frac{1}{T} \displaystyle\sum_{t=1}^T\mathbb{E}_{(x,y)\sim{\mu_t}}[\mathcal{L}(\langle w_t,x\rangle,y)]
\end{split}
\label{eq:expected risk}
\end{equation}
Given the data $Z=(X,Y)$, the empirical risk is defined as:  
\vspace{-0.1cm}s
\begin{equation}
\begin{split}
\mathbb{\hat{E}}(w|Z) \coloneqq \frac{1}{T} \displaystyle\sum_{t=1}^T\frac{1}{m}\displaystyle\sum_{i=1}^m\mathcal{L}(\langle w_t,X_{ti}\rangle,Y_{ti})
\end{split}
\label{eq:empirical risk}
\end{equation}
\end{definition}

\begin{definition}
(Global optimal solution, Optimized solution).
Define $\mathcal{F}_{\alpha,\beta} = \{w\in\mathbb{R}^{d \times T}: \displaystyle\sum_{t=1}^T\left\lVert w_t \right\rVert_{1} \leq \alpha, \displaystyle\sum_{t=1}^{T-1}\left\lVert max(0,-w_{t} \otimes w_{t+1}) \right\rVert_{1} \leq \beta \}$. Denote $w^*$ as the global optimal solution of the expected risk:
\vspace{-0.1cm}
\begin{equation}
\begin{split}
w^* & \coloneqq \arg\min_{w\in\mathcal{F}_{\alpha,\beta}}\mathbb{E}(w) = \arg\min_{w\in\mathcal{F}_{\alpha,\beta}}\frac{1}{T} \displaystyle\sum_{t=1}^T\mathbb{E}_{(x,y)\sim{\mu_t}}[\mathcal{L}(\langle w_t,x_{ti}\rangle,y_{ti})]
\end{split}
\label{eq:global optimal solution}
\end{equation}
Denote $w_{(Z)}^*$ as the optimized solution by minimizing the empirical risk:
\vspace{-0.1cm}
\begin{equation}
\begin{split}
w_{(Z)}^* & \coloneqq \arg\min_{w\in\mathcal{F}_{\alpha,\beta}}\mathbb{\hat{E}}(w|Z) = \arg\min_{w\in\mathcal{F}_{\alpha,\beta}}\frac{1}{T} \displaystyle\sum_{t=1}^T\frac{1}{m}\displaystyle\sum_{i=1}^m\mathcal{L}(\langle w_t,X_{ti}\rangle,Y_{ti})
\end{split}
\label{eq:optimized solution}
\end{equation}
\end{definition}

Finally, the following theorem shows the upper-bounded generalization error of our SRML model.

\begin{theorem}[Generalization error bound]
\label{thm:error bound}
Let $\epsilon>0$ and $\mu_1,\mu_2,\dots,\mu_T$ be the probability measure on $\mathbb{R}^d \times\mathbb{R}$. With probability of at least $1-\epsilon$ in the draw of $Z=(X,Y)\sim\prod_{t=1}^{T}{\mu_{t}^{m}}$, we have:
\vspace{-0.1cm}
\begin{equation}
\begin{aligned}
\mathbb{E}(w_{(Z)}^*)-\mathbb{E}(w^*) & = \frac{1}{T} \displaystyle\sum_{t=1}^T\mathbb{E}_{(x,y)\sim{\mu_t}}[\mathcal{L}(\langle w_{(Z)t}^*,x\rangle,y)] - \inf_{w\in\mathcal{F}_{\alpha,\beta}} \frac{1}{T} \displaystyle\sum_{t=1}^T\mathbb{E}_{(x,y)\sim{\mu_t}}[\mathcal{L}(\langle w_t,x\rangle,y)] \\
& \leq \frac{2L\alpha}{mT}\max_{1\leq t \leq T}\left\lVert x_t \right\rVert_{1,\infty} + 2\sqrt{\frac{2\ln{2/\epsilon}}{mT}}
\end{aligned}
\label{eq:generalizatio error bound}
\end{equation}
\end{theorem}
\vspace{-0.4cm}
\begin{proof}
Due to the limited space, we put the proof in \autoref{error bound proof} in supplementary.
\end{proof}
\vspace{-0.2cm}
This theorem provides important insights into the proposed model: 1) The more training samples 
used, the less generalization error it will be; 2) The generalization error converges to 0 when the training sample size approaches infinity; 3) The smaller value of $\max_{1\leq t \leq T}\left\lVert x_t \right\rVert_{1,\infty}$ is, the faster convergence rate of the error bound will be.
Notice that the hyperparameter $\beta$ is not included in the bound, which means the level of slacking in our SWMTL model doesn't affect the bound. A high-level reason is, the hyperparameter $\beta$ only controls the sign of the weights but couldn't control their magnitude. In mathematics, consider the following result:
\vspace{-0.1cm}
\begin{equation}
\begin{split}
\mathbb{E}\{\sup_{w\in\mathcal{F}_{\alpha,\beta}}\displaystyle\sum_{t=1}^T\displaystyle\sum_{i=1}^m\sigma_{ti}\langle w_t, x_{ti}\rangle \} & = \alpha\cdot\mathbb{E}\{\max_{1\leq t \leq T, 1\leq j \leq d} |\displaystyle\sum_{i=1}^m x_{tij}\sigma_{ti}| \} \\
& \leq  \alpha\cdot\max_{1\leq t \leq T}\left\lVert x_t \right\rVert_{1,\infty}
\end{split}
\label{eq:beta discussion}
\end{equation}
The first equation is because the function inside $\sup$ is linear w.r.t. $w$, and under the constraint $\mathcal{F}_{\alpha,\beta}$ we will see the weight with the largest absolute value of coefficient (suppose it is unique) equals $\alpha$ and all the others equal 0. For more comprehensive explanation, please refer to \autoref{eq:lemma1proof} in \autoref{error bound proof} of supplementary.

\subsection{Convergence Analysis}
\label{convergence analysis}

In this section, we analyze the conditions and properties of the convergence of our optimization algorithm. We prove the convergence by first giving the following definition.

\begin{definition}
Given any input data $X\in (\mathbb{R}^d)^{mT}$, define constant $H_{reg}$ and $H_{class}$ as:
\vspace{-0.1cm}
\begin{equation}
\begin{split}
H_{reg} \coloneqq \max_t\{2\left\lVert X_t^\intercal X_t\right\rVert \} \text{;} \ \ \ \ H_{class} \coloneqq \max_t\{\frac{1}{m}\displaystyle\sum_{j=1}^{m} \left\lVert X_{tj} \right\rVert^{2} \} 
\end{split}
\label{eq:H1H2}
\end{equation}
\end{definition}
\vspace{-0.3cm}
Given a problem of regression or classification, set $H$ equals to the corresponding one and we have:

\begin{theorem}[Global convergence]
\label{thm:convergence}
If $\rho > 2H$, then for the variables $(w_1,\cdots,w_t,u_1,\cdots,u_t,y)$ in  \autoref{eq:aug_lagrangian_swmtl_3}, starting from any $(w_1^0,\cdots,w_t^0,u_1^0,\cdots,u_t^0,y^0)$, this sequence generated by miADMM has the following properties: 
\begin{enumerate*}
    \item Dual convergence: $y^k$ converges as $k\rightarrow\infty$. 
    \item Residual convergence: $r^k\rightarrow 0$ and $s^k\rightarrow 0$ as $k\rightarrow\infty$, where $r$ and $s$ are the primal and dual residual.
    \item Objective convergence: the whole objective function defined in \autoref{eq:obj_swmtl_3} converges as $k\rightarrow\infty$.
\end{enumerate*}
\end{theorem}
\vspace{-0.5cm}
\begin{proof}
The proof, which is very technical, can be found in \autoref{appendix:convergence_analysis} in supplementary.
\end{proof}

\vspace{-0.3cm}
\autoref{thm:convergence} only guarantees the convergence of the ADMM algorithm, but $w$ and $u$ are not necessarily converging. However, by the Theorem 2~\cite{wang2019multi}, we can show that they will converge to a Nash point.
\begin{theorem}[Convergence to a Nash point]
\label{thm:Converge Nash point}
For $w$ and $u$ defined in \autoref{eq:obj_swmtl_3}, $(w_1^k,\cdots,$ $w_T^k,u_1^k,\cdots,u_T^k)$ will converge to a feasible Nash point $(w_1^*,\cdots,w_T^*,u_1^*,\cdots,u_T^*)$ of the objective function defined in the corresponding problem, i.e.: \\
(Feasibility) $[w_1^*;\cdots;w_T^*]-[u_1^*;\cdots;u_T^*]=0 $ \\
(Nash point) $F(w^*,u^*) \leq F(w_1^*,\cdots,w_{t-1}^*,w_t,w_{t+1}^*,\cdots,w_T^*,u^*), \ \ \forall{(w_1^*,\cdots,w_{t-1}^*,w_t,w_{t+1}^*,}$ $\cdots,w_T^*,u^*)\in dom(F)$; \ \  $F(w^*,u^*)\leq F(w^*,u_1^*,\cdots,u_{t-1}^*,u_t,u_{t+1}^*,\cdots,u_T^*),
\ \ \forall{(w^*,u_1^*,\cdots,}$ $u_{t-1}^*,u_t,u_{t+1}^*,\cdots,u_T^*)\in dom(F)$.
\end{theorem}
\vspace{-0.5cm}
\begin{proof}
The proof can be found in \autoref{appendix:convergence_analysis} in supplementary.
\end{proof}
\vspace{-0.3cm}
The last theorem in this section shows the convergence rate of our miADMM algorithm. 

\begin{theorem}[Convergence rate of algorithm]
\label{thm:converge rate}
For a sequence $(w_1^k,\cdots,w_T^k,u_1^k,\cdots,u_T^k,y^k)$, define
\begin{equation}
\begin{split}
v^k = \min_{0\leq l\leq k}(\displaystyle\sum_{t=1}^{T}\left\lVert w_t^{l+1}-w_t^{l} \right\rVert_{2}^{2} +\displaystyle\sum_{t=1}^{T}\left\lVert u_t^{l+1}-u_t^{l}\right\rVert_{2}^{2})
\end{split}
\label{eq:vk}
\end{equation}
then the convergence rate of $v_k$ is $o(1/k)$.
\end{theorem}
\vspace{-0.5cm}
\begin{proof}The proof can be found in \autoref{appendix:convergence_analysis} in supplementary.
\end{proof}

\vspace{-0.3cm}
\subsection{Time complexity analysis}
\label{time complexity}
In Algorithm 1, we denote the number of iterations for miADMM as $l_1$ and the number of iterations for (projected) gradient descent as $l_2$. The time complexity for one iteration of gradient descent for subproblem of $w_t$ should be the time complexity for calculating the gradient for $w_t$. For example, in regression problem with $L_2$ regularization, the gradient for $w_t$ is $2(X_t^{\intercal}X_t+(\lambda+\rho)I)w_t - (2X_t^{\intercal}Y_t+\rho(u_t^k-y_t^k/\rho))$. The time complexity for calculating this gradient should be $\mathcal{O}(dm)$. In addition, we analytically solve the subproblem for a single $u_{tj}$ (which is a scalar), where we find a minima for a univariate piece-wise quadratic function. Therefore, the time complexity for solving all subproblems of $u$ should be $\mathcal{O}(Tdm)$. Hence, the total time complexity of our miADMM algorithm is $\mathcal{O}(l_{1}(l_{2}T(dm)+Tdm))$ and can be simplified as $\mathcal{O}(l_{1}l_{2}Tdm)$, which is linear to the sample size.

\input{section_experiments}

\input{section_conclusions}

\small
\bibliographystyle{plainnat}
\bibliography{main}

\clearpage

\appendix

\input{appendix_optmethod}

\input{appendix_theoretical}

\input{appendix_experiments}

\end{document}

%% file: section_experiments.tex
\section{Experiments}
In this section, the performance of our SRML framework is evaluated using several synthetic and real-world datasets against the state-of-the-art, on various aspects including accuracy, efficiency, convergence, sensitivity, scalability, and qualitative analyses. The experiments were performed on a 64-bit machine with a 8-core processor (i9, 2.4GHz), 64GB memory.

\subsection{Experimental Settings}
\textbf{Synthetic Datasets}: There are 3 synthetic datasets named Synthetic Dataset 1, Synthetic Dataset 2, and Synthetic Dataset 3, whose generation process is elaborated in the following.
We generate $T$ tasks ($T = 20$) and for each task $i$ we generate $m$ samples ($m=100$).
The input data $X_i \in \mathbb {R}^{m \times d}$ for each task $i$ is generated from $\displaystyle X_i \sim {\mathcal {N}}(\mu ,I) + \eta_i$, with mean  $\mu=\mathbf{0}$. Here $\eta_i \sim {\mathcal {U}}(\mathbf{0} ,\mathbf{10}) $ represents the bias values of the features for $i$-th task. Next, we generate feature weight, following three steps: 1) Generate the polarity of features. We define $P\in\{0,-1,1\}^d$ as the signs of all the features, which is generated by obtaining the signs of a $d$-dimension vector sampled from an isotopic Gaussian ${\mathcal {N}}(\mathbf{0}, I)$. 2) Generate the feature weights. For each task $i$, we first calculate the weight magnitude $\tilde W_i\in\mathbb{R}^d$ which is the absolute value of a randomly sampled vector from an isotropic Gaussian ${\mathcal {N}}(\mathbf{0}, I)$. Then the final feature weight is calculated by assembling the sign and magnitude by $W_i = P \otimes \tilde W_{i}$, where $\otimes$ is element-wise multiplication. 3) Add noise the weight matrix. We add noise to the weight matrix $W_i$ for each task $i$ by randomly flipping the sign of $10\%$ of the weights. The generation process of target variable differs across different synthetic datasets.
For regression problem, the target variable of the $i$-th task is generated as: $Y_i = X_i\cdot W_i + \epsilon_i$, where $\epsilon_{i} \sim \mathcal {N} (\mathbf{0}, 0.1\cdot I)$. For classification problem, the target variable is generated as $Y_i =  \sigma(X_i\cdot W_i + \epsilon_i)$, where $Y_i=1$ if $Y_i \geq 0.5$ and $Y_i=0$ $o.w.$, and $\sigma$ is the sigmoid function. \textbf{Synthetic Dataset 1} is for regression, with 20 tasks, 100 instances per task, and 25 features. \textbf{Synthetic Dataset 2} is for regression, with 100 tasks, 100 instances per task, and 1000 features. \textbf{Synthetic Dataset 3} is for classification, with 5 tasks, 100 instances per task, and 25 features. 

\textbf{Real-World Datasets: }Five real-world datasets were used to evaluate the proposed methods and the comparison methods, including: \textbf{1) School Dataset~\cite{10.2307/2348496}, 2) Computer Dataset~\cite{lenk1996hierarchical}, 3) Facebook Metrics Dataset~\cite{moro2016predicting}, 4) Traffic SP Dataset~\cite{Ferreira2016CombinationOA}, and 5) Cars Dataset~\cite{lock19931993}.} Their detailed descriptions and download links are elaborated in Section \ref{sec:real_world_datasets} of our supplementary material.

\textbf{Comparison Methods and Baseline.}
Existing sparse feature learning and multitask learning methods have been included to compare the performance with the proposed SRML models.
\begin{enumerate*}[label=(\arabic*)]
\item Lasso is an $\ell_1$-norm regularized method which introduce sparsity into the model to reduce model complexity and feature learning, and that the parameter controlling the sparsity is shared among all tasks~\cite{tibshirani1996regression}.
\item Join Feature Learning (L21) assumes the tasks share a set of common features that represent the  relatedness of multiple tasks~\cite{argyriou2007multi}.
\item Convex alternating structure optimization (cASO) decomposes the predictive model of each task into two components: the task-specific feature mapping and task-shared feature mapping~\cite{chen2009convex}.
\item Robust multi-task learning (RMTL) method assumes that some tasks are more relevant than others. It assumes that the model $W$ can be decomposed into a low rank structure $L$ that captures task-relatedness and a group-sparse structure $S$ that detects outliers~\cite{chen2011integrating}. 
\item Sparse Structure-Regularized Multi-task Learning (SRMTL) represents the task relationship   using a graph where each task is a node, and two nodes are connected via an edge if they are related~\cite{evgeniou2004regularized}. \item Strict Sign-regularized Multi-task learning (SSML) is our method's baseline version which follows Equation \eqref{eq:wmtl}, without slack variable in our SRML to enhance robustness. 
\end{enumerate*}

\textbf{Evaluation Metrics:} 
To evaluate the performance of the methods for the regression problem, we employ the mean absolute error ($MAE$), mean square error ($MSE$), and the mean square logarithmic error ($MSLE$).
Note that better regression performance is indicated by the smaller value of $MSE$ or $MAE$. 
For the classification problem, the accuracy (ACC), area under the curve score (AUC), and mean precision average (MAP) are used to evaluate the performance, where a larger value denotes better performance. 

\textbf{Hyperparameter Tuning:}
Each task data is split into $60\%$ (for training) and $40\%$ (for testing).
For hyper-parameters tuning of our method and all the comparison methods, cross validation is applied on the training set via 5-fold cross validation and grid search (logarithmic search on the range $\{10^{-3} \dots 10^{3} \}$), particularly for the values of regularization terms. 

\subsection{Experimental Results}

\vspace{-0.3cm}

\begin{table*}[htbp]
  \centering
  \scriptsize
  \caption{Performance on real-world datasets (MSE)}
  \vspace{-0.2cm}
    \begin{tabular}{lp{0.4cm}lp{0.25cm}lp{0.85cm}lp{0.85cm}lp{0.20cm}lr}
    \toprule
    \textbf{Model} & \multicolumn{2}{c}{\textbf{School}} & \multicolumn{2}{c}{\textbf{Computers}} & \multicolumn{2}{c}{\textbf{Cars}} & \multicolumn{2}{c}{\textbf{Facebook}} &  \multicolumn{2}{c}{\textbf{TrafficSP}}&Runtime \\
    \midrule
    CASO  &   107.39  & $\pm$1.65 & 31.91 & $\pm$5.21 & 2.36E+08 & $\pm$1.62E+08 & 1.50E+05 & $\pm$8.46E+04  &   9.83  & $\pm$1.32 &45.49 seconds\\
    L21   &   107.78  & $\pm$1.66 & 31.91 & $\pm$5.21 & 2.09E+08 & $\pm$1.34E+08 & 1.52E+05 & $\pm$8.13E+04 &   10.46  & $\pm$1.21 & 5.81 seconds\\
    LASSO &   108.30  & $\pm$1.65 & 31.91 & $\pm$5.21 & 2.09E+08 & $\pm$1.34E+08 & 1.52E+05 & $\pm$8.13E+04 &   10.46  & $\pm$1.21 & 2.84 seconds\\
    RMTL  &   108.16  & $\pm$1.65 & 39.89 & $\pm$7.11 & 2.12E+08 & $\pm$1.34E+08 & 1.53E+05 & $\pm$8.03E+04  &   10.24  & $\pm$1.38 & 12.09 seconds\\
    SSML  &   107.89  & $\pm$1.56 & 31.93 & $\pm$5.21 & 3.34E+08 & $\pm$3.97E+08 & 1.51E+05 & $\pm$8.03E+04 &     9.70  & $\pm$1.23 &8.62 seconds\\
    SRML  & \textbf{106.65} & $\pm$1.90 & \textbf{30.63} & $\pm$5.78 & \textbf{1.89E+08} & $\pm$1.05E+08 & \textbf{1.49E+05} & $\pm$8.59E+04  & \textbf{9.52} & $\pm$1.10 &7.87 seconds\\
    \bottomrule
    \end{tabular}%
  \label{tab:results_reg_realdatasets}%
  \vspace{-0.2cm}
\end{table*}

\begin{table*}
\scriptsize
\centering
\parbox{.49\linewidth}{
\caption{Perf. on Synthetic Dataset 1. \label{tab:results_reg_task_syntheticdataset}}
\vspace{-0.3cm}
\begin{tabular}{lccc}
    \toprule
    MODEL & \multicolumn{1}{c}{MAE} & \multicolumn{1}{c}{MSE} & \multicolumn{1}{c}{MSLE} \\
    \midrule
    CASO  & 6.96E+01 & 8.55E+03 & 9.21E-03 \\
    L21   & 7.12E+01 & 8.96E+03 & 6.23E-03 \\
    LASSO & 7.12E+01 & 8.96E+03 & 6.23E-03 \\
    RMTL  & 6.80E+01 & 8.14E+03 & 1.13E-02 \\
    SSML  & 2.73E+03 & 1.28E+07 & 1.20E+00 \\
    SRML & \textbf{2.02E+00} & \textbf{1.05E+01} & \textbf{1.36E-06} \\
    \bottomrule
    \end{tabular}
}
\hspace{-1.5cm}
\parbox{.49\linewidth}{
\caption{Perf. on Synthetic Dataset 2. \label{tab:results_1k_syntheticdataset}}
\vspace{-0.3cm}
\begin{tabular}{lccccc}
    \toprule
    MODEL & \multicolumn{1}{c}{NMAE} & \multicolumn{1}{c}{NMSE} & \multicolumn{1}{c}{MAE} & \multicolumn{1}{c}{MSE} & \multicolumn{1}{c}{MSLE} \\
    \midrule
    CASO  & 0.0866 & 2727.7 & 1.9656E+4 & 6.1854E+8 & 7.89 \\
    L21   & 0.0856 & 2671.3 & 1.9426E+4 & 6.0574E+8 & 6.90 \\
    LASSO & 0.0856 & 2671.3 & 1.9426E+4 & 6.0574E+8 & 6.90 \\
    RMTL  & 0.0856 & 2671.3 & 1.9426E+4 & 6.0574E+8 & 6.90 \\
    SSML  & 0.0873 & 2764.2 & 1.9799E+4 & 6.2681E+8 & 6.98 \\
    SRML & \textbf{0.0856} & \textbf{2671.1} & \textbf{1.9425E+4} & \textbf{6.0571E+8} & \textbf{6.87} \\
    \bottomrule
    \end{tabular}
}
\vspace{-0.2cm}
\end{table*}


\begin{table}[h!]
\scriptsize
\centering
\caption{Perf. on Synthetic Dataset 3 \label{tab:results_clf_task_syntheticdataset}}
\vspace{-0.2cm}
\begin{tabular}{lccc}
    \toprule
    MODEL & \multicolumn{1}{c}{ACC} & \multicolumn{1}{c}{AUC} & \multicolumn{1}{c}{MAP} \\
    \midrule
    CASO  &            82.40  &            82.96  &              82.96  \\
    L21   &            81.30  &            81.85  &              81.79  \\
    LASSO &            82.70  &            83.09  &              83.06  \\
    SRMTL  &            81.35  &            81.99  &              81.82  \\
    SRML & \textbf{           82.95 } & \textbf{           90.99 } & \textbf{             84.72 } \\
    \bottomrule
    \end{tabular}
\vspace{-0.7cm}
\end{table}


\textbf{Effectiveness Evaluation in Synthetic Datasets: }The empirical results show that our SRML model achieves the best performance on synthetic datasets for regression task (\autoref{tab:results_reg_task_syntheticdataset}) and classification task (\autoref{tab:results_clf_task_syntheticdataset}). 
In the case of the regression task, our SRML model outperforms the baseline models by a large margin for all the metrics. For the MAE, it achieves an order of magnitude better score w.r.t the best baseline model RMTL. The MSE and MSLE metrics show similar improvements (several orders of magnitude w.r.t the baseline model). Although SSML uses a similar approach, it enforces a strict polarity regularization compared to our model. The hard constraints in the SSML model fail to capture changes in features' polarity between tasks and achieve the worst performance compared to other baseline models. For the binary classification task, our model achieves the best score in every metric (ACC, AUC, MAP). The AUC metric shows a significant margin ($8\%$) compared to the baseline, which indicates that our model will perform better at different thresholds for labels.
However, the margin of improvement is small for ACC and MAP metric w.r.t. the best comparison method. The reason for the small margin improvement on ACC and MAP for the classification task is because the dependent variable is less sensitive to the variation of the magnitude of weights of the model parameters in the dataset generation. The SSML has been excluded from classification experiments as the reference paper only provide their implementation for regression tasks.

For the Synthetic Dataset 2, our SRML with $L_1$ regularization achieves the best score in each metric shown in \autoref{tab:results_1k_syntheticdataset}, where the dimension for features is 1,000. NMAE and NMSE stands for Normalized Mean Absolute Error and Normalized Mean Square Error respectively \cite{bishop2006pattern}, which is the MAE and MSE normalized by the range of ground-truth label. We notice the margin for our SRML in high-dimensional case is smaller than that in \autoref{tab:results_reg_task_syntheticdataset}. This is possibly because for most of the features with zero weights, the sparsity regularization might be already enough. But the sign constraint is still important for those nonzero features and does not harm those with zero weights.

In order to investigate whether and how the proposed sign regularization approaches in SRML impact and benefit the learning of feature weight signs, \autoref{fig:feature_selection} shows a comparison among different methods in terms of the difference between their learned signs and ground truth signs in all tasks and features. The first subplot labeled ``syntheticWMTLR4W'' shows the ground truth feature weights' signs, while the other subplots correspond to the differences between the weights' signs learned by different models and the ground truth signs shown in the first subplot. It can be clearly seen that our SRML achieves an exact match to the ground truth as the cells are all-white, meaning ``no difference'' to the ground truth. It hence outperforms the competing methods who do not leverage the sign-regularization for instructing multitask learning for this types of tasks. Moreover, as we expected, the baseline SSML has numerous cells different to the ground truth because it leverage strict constraints forcing each feature's weights to be the same across tasks. Therefore, the effectiveness of our ``slack mechanism'' is clearly shown by contrasting the performance between SSML and SRML.

\textbf{Effectiveness Evaluation in Real-world Datasets:} \autoref{tab:results_reg_realdatasets} shows the results of our method SRML and the other methods on the 5 real-world datasets in the MSE metric (average over 10 runs and the standard deviation). Our model SRML model outperforms the comparison methods on all the datasets, by a clear margin. In the Cars dataset our model outperform with a considerable margin ($9\%$  w.r.t to the 2nd best model L21) while in Computer Dataset we outperform by around $3\%$.  
We found our model performs well on datasets (e.g., Computers and Cars) with mixed features types (categorical and real values), since in these types of datasets we can exploit the features' sign correlation between different tasks. 
The second best method in general is our baseline SSML which also involves sign-regularization but without slack to absorb noise in real-world data. 
LASSO is relatively weak in most of the datasets due to its incapability of utilizing the relationship among tasks. 
In addition, the training runtime on TrafficSP Dataset is also presented, where we can see that the fastest method is LASSO due to its relative simplicity. Our method, though is slower than simple methods such as LASSO and L21, is still highly efficient comparing with other complex methods such as CASO and RMTL. The runtime on other datasets follow similar trend.

\textbf{Convergence Analysis: }
The trends of objective function value, primal and dual residual during the optimization of one training process are illustrated in  \autoref{fig:obj_conv}, \autoref{fig:primal_conv}, and \autoref{fig:dual_conv}, respectively. They demonstrate the convergence of all of them,which is consistent with the convergence analysis in Section \ref{convergence analysis}.

\begin{figure*}[h!]
\vspace{-1cm}
\begin{center}
\begin{subfigure}[b]{.3\textwidth}
\scalebox{.8}{
\includegraphics[width=\textwidth]{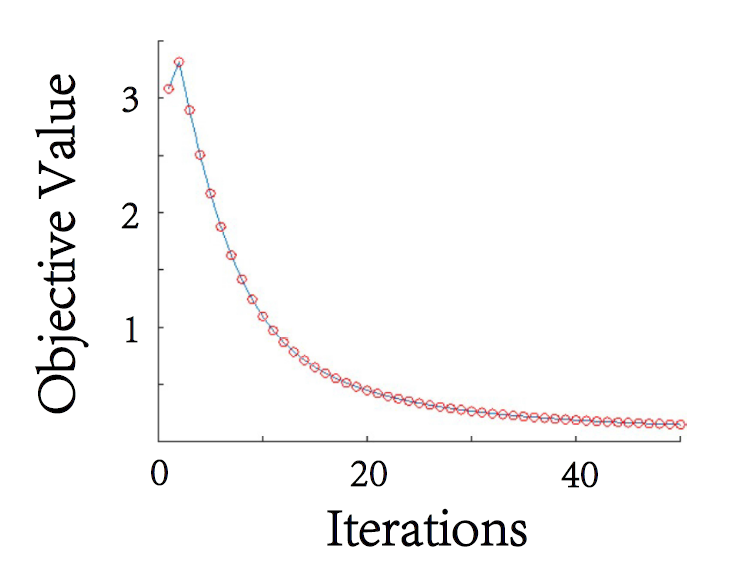}
}
\caption{Iteration v.s. objective value}
\vspace{-0.2cm}
\label{fig:obj_conv}
\end{subfigure}
\begin{subfigure}[b]{.3\textwidth}
\scalebox{.8}{
\includegraphics[width=\textwidth]{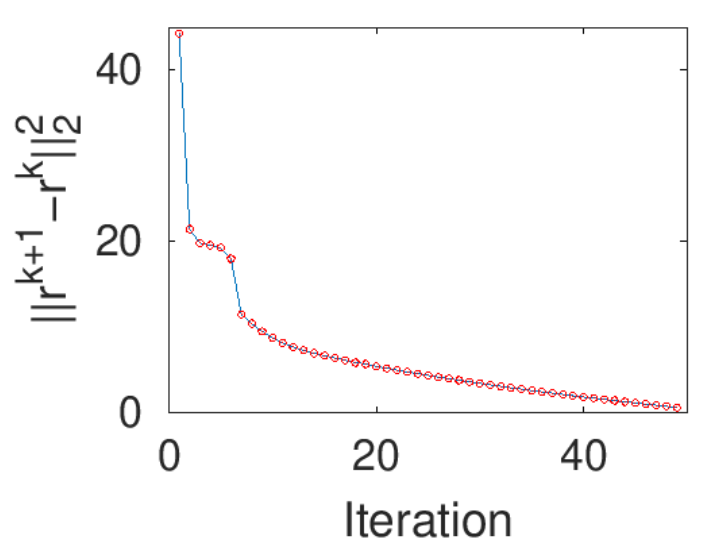}
}
\caption{Iteration v.s. primal residual}
\vspace{-0.2cm}
\label{fig:primal_conv}
\end{subfigure}
\begin{subfigure}[b]{.3\textwidth}
\scalebox{.8}{
\includegraphics[width=\textwidth]{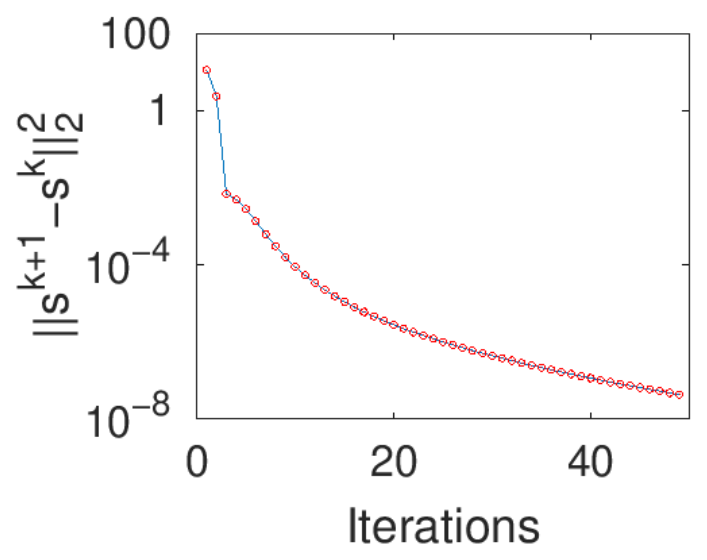}
}
\caption{Iteration v.s. dual residual}
\vspace{-0.2cm}
\label{fig:dual_conv}
\end{subfigure}
\caption{Convergence on synthetic dataset.}
\label{fig:convergence_curves}
\end{center}
\end{figure*}

\begin{figure*}[h!]
\vspace{-0.5cm}
\begin{center}
\begin{subfigure}[b]{.3\textwidth}
\scalebox{.8}{
\includegraphics[width=\textwidth]{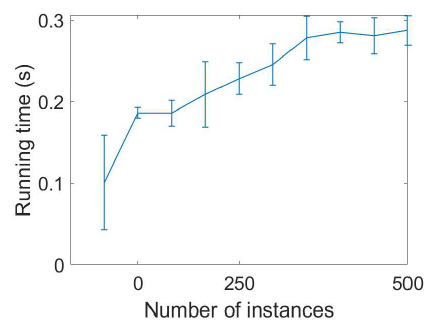}
}
\caption{Runtime v.s. \#Instances}
\vspace{-0.2cm}
\label{fig:scalability_m}
\end{subfigure}
\begin{subfigure}[b]{.3\textwidth}
\scalebox{.8}{
\includegraphics[width=\textwidth]{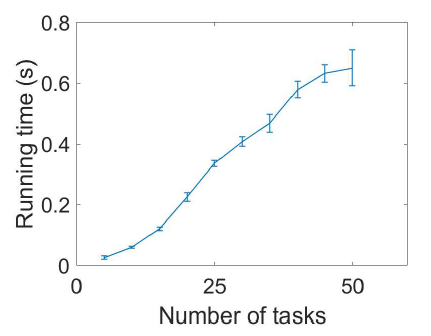}
}
\caption{Runtime v.s. \#Tasks}
\vspace{-0.2cm}
\label{fig:scalability_t}
\end{subfigure}
\begin{subfigure}[b]{.3\textwidth}
\scalebox{.8}{
\includegraphics[width=\textwidth]{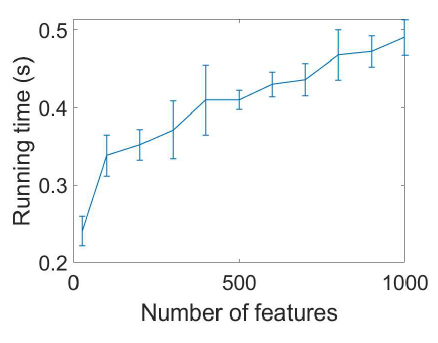}
}
\caption{Runtime v.s. \#Features}
\vspace{-0.2cm}
\label{fig:scalability_d}
\end{subfigure}
\caption{Scalability Analysis on synthetic dataset.}
\label{fig:scalability_curves}
\end{center}
\vspace{-0.2cm}
\end{figure*}

\begin{figure*}[h!]
\vspace{-0.2cm}
\begin{center}
\begin{subfigure}[b]{.25\textwidth}
\scalebox{.8}{
\includegraphics[width=\textwidth]{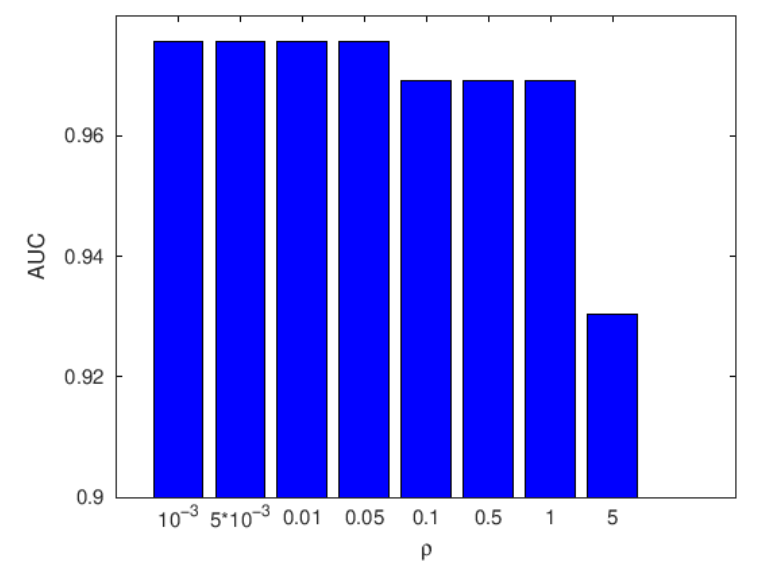}
}
\caption{Sensitivity of parameter $\rho$}
\vspace{-0.2cm}
\label{fig:sensitivity_curves_p}
\end{subfigure}
\begin{subfigure}[b]{.25\textwidth}
\scalebox{.8}{
\includegraphics[width=\textwidth]{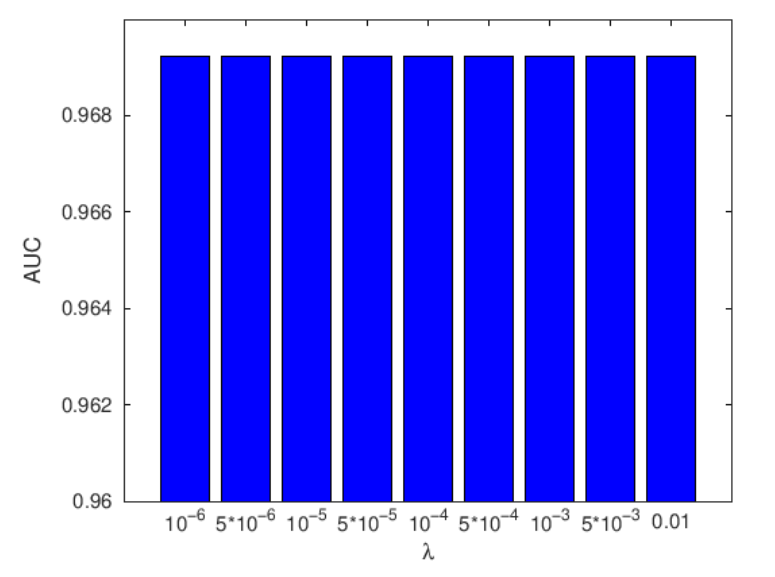}
}
\caption{Sensitivity of parameter $\lambda$}
\vspace{-0.2cm}
\label{fig:sensitivity_curves_l}
\end{subfigure}
\begin{subfigure}[b]{.25\textwidth}
\scalebox{.8}{
\includegraphics[width=\textwidth]{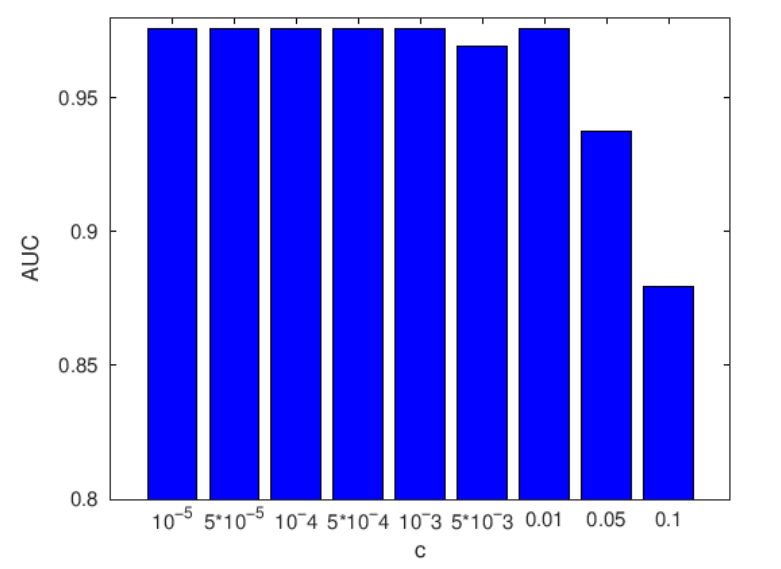}
}
\caption{Sensitivity of parameter c}
\vspace{-0.2cm}
\label{fig:sensitivity_curves_c}
\end{subfigure}
\caption{Sensitivity Analysis on synthetic dataset.}
\label{fig:sensitivity_curves}
\end{center}
\vspace{-0.2cm}
\end{figure*}

\begin{figure*}[h!]
\centering
\vspace{-0.3cm}
\includegraphics[width=\textwidth]{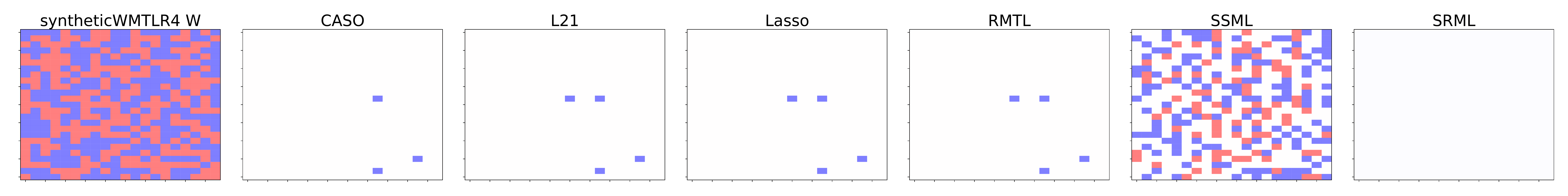}
\vspace{-0.8cm}
\caption{Illustration on how well the signs of the learned feature weights match the ground truth on Synthetic Dataset 1.
For each model, we show the selected weights ($y$ axis) of each of the 20 tasks ($x$ axis), where each cell's color  denote if the sign matches to the ground truth (\textit{white} color), or there is a difference either positive (\textit{red}) or negative (\textit{blue}). Hence, our model's results completely match the ground truth.
}
\label{fig:feature_selection}
\vspace{-0.2cm}
\end{figure*}

\textbf{Scalability Analysis: }\autoref{fig:scalability_curves} illustrates the scalability of the proposed SRML model in the regression synthetic dataset in the training running time when the size of the dataset varies. Each setting of synthetic dataset was generated randomly for ten times and thus the standard deviation was calculated and shown by the error bar. Specifically, \autoref{fig:scalability_m} shows that when the numbers of tasks and features are fixed, the runtime increases near-linearly when the number of instances increases. In addition, \autoref{fig:scalability_t} shows that when the number of instances and features are fixed, the runtime also increases linearly when the number of tasks increases. In the last one, which is \autoref{fig:scalability_d}, when the numbers of tasks and instances are fixed, the runtime increases linearly when the number of features increases. All the observations are consistent with the theoretical analysis on time complexity in Section \ref{time complexity}.

\textbf{Sensitivity Analysis: }
The sensitivity of hyperparameters of the proposed SRML on the classification synthetic dataset is illustrated in \autoref{fig:sensitivity_curves}. \autoref{fig:sensitivity_curves_p} illustrates that our model performs best with parameter $\rho$ smaller than 0.1. In addition, \autoref{fig:sensitivity_curves_l} shows our model is barely sensitive to the coefficient regularization $\lambda$ \autoref{fig:sensitivity_curves_l}. This is potentially reasonable because the synthetic dataset for this experiment has low dimensions in features and no sparsity. Last, \autoref{fig:sensitivity_curves_c} shows our SRML model performs best when the parameter $c$ is smaller than 0.1. This makes sense because we added noise into the sign of ground truth weights and since smaller $c$ provides SRML more slacking our model could achieve better score.

%% file: section_conclusions.tex
\section{Conclusions}

Considering the assumption that in some real-world applications, the tasks share a similar polarity for features across tasks, we propose sign-regularized multi-task learning framework by enforcing the learning weights to share polarity information to neighbors tasks. Experiments on multiple synthetic and real-world datasets demonstrate the effectiveness and efficiency of our methods in various metrics, compared with several comparison methods and baselines. Various analyses such as convergence analyses, scalability have also been done theoretically and experimentally. Additional analyses on the learned parameters such as sensitivity analyses and qualitative analyses on learned parameters have also been discussed.


%% file: appendix_optmethod.tex
{\Large Supplementary Materials}


\section{Optimization Method}
\label{optimization method}
The detailed optimization algorithm of our SRML model is shown in Algorithm \ref{alg:swmtl} as follow.

\begin{algorithm}[H]
   \caption{ADMM Algorithm to Solve \autoref{eq:aug_lagrangian_swmtl_3}}
   \label{alg:swmtl}
   \begin{algorithmic} 
   \STATE Denote $w=[w_1;\cdots;w_T]$,  $u=[u_1;\cdots;u_T]$
   \STATE Denote $y=[y_1;\dots;y_T]$
   \STATE Initialize $\rho,c, k=0$
   \REPEAT
   \FOR{$t=1$ {\bfseries to} $T$}
   \STATE  $w_{t}^{k+1}\leftarrow$ solve \autoref{eq:subproblem_w_swmtl_3}
   \ENDFOR
   \STATE  $u_{1}^{k+1} \leftarrow$ solve \autoref{eq:subproblem_u_swmtl_3} 
   \FOR{$t=2$ {\bfseries to} $T-1$}
   \STATE $u_{t}^{k+1}\leftarrow$ solve \autoref{eq:subproblem_n-1_swmtl_3}
   \ENDFOR
   \STATE $u_{T}^{k+1}\leftarrow$  solve \autoref{eq:subproblem_t_swmtl_3}
   \STATE $y^{k+1}\leftarrow\, y^k + \rho(w^{k+1}-u^{k+1})$
   \STATE $ k\leftarrow\, k+1$.
   \UNTIL{convergence}
   \STATE Output $w$, \ \ $u$.
\end{algorithmic}
\end{algorithm}

For each iteration $k$, the $T$ subproblems for updating $w_t$'s are as follow: \\
For $t = 1,2,\cdots,T$,

\begin{equation}
\begin{split}
w_{t}^{k+1}\leftarrow \, &\arg\min_{w_t} \mathcal{L}_t(w_t) + \lambda\Omega(\{w_t\}_t^\intercal) + (\rho/2)\left\lVert w_t-u_{t}^k+y_{t}^{k}/\rho\right\rVert_{2}^{2}
\end{split}
\label{eq:subproblem_w_swmtl_3}
\end{equation}

where we can use gradient descent when $\Omega(\cdot)$ is differentiable and projected gradient descent, otherwise.

The $T$ subproblems for updating $u_i$'s are as follow: 

\begin{equation}
\begin{split}
u_{1}^{k+1}\leftarrow \, &\arg\min_{u_1} c \displaystyle\sum_{j=1}^d \max {(0,-u_{1,j}u_{2,j}^k)} - (y_1^{k})^\intercal u_1 + (\rho/2)\left\lVert w_1^{k+1}-u_{1}\right\rVert_{2}^{2}  
\end{split}
\label{eq:subproblem_u_swmtl_3}
\end{equation}

For $t=2,3,\dots,T-1$:

\begin{equation}
\begin{split}
u_{t}^{k+1}\leftarrow\, & \arg\min_{u_t} c\displaystyle\sum_{j=1}^d [ \max {(0,-u_{t,j}u_{t-1,j}^{k+1})} + \max {(0,-u_{t,j}u_{t+1,j}^k)} ] - (y_t^{k})^\intercal u_i + (\rho/2)\left\lVert w_t^{k+1}-u_{t}\right\rVert_{2}^{2}
\end{split}
\label{eq:subproblem_n-1_swmtl_3}
\end{equation}

For $t=T$, we have:
\begin{equation}
\begin{split}
u_{T}^{k+1}\leftarrow\, &\arg\min_{u_T} c \displaystyle\sum_{j=1}^d \max {(0,-u_{T,j}u_{T-1,j}^{k+1})} - (y_T^{k})^\intercal u_T + (\rho/2)\left\lVert w_T^{k+1}-u_T\right\rVert_{2}^{2}  
\end{split}
\label{eq:subproblem_t_swmtl_3}
\end{equation}

For each subproblem of $u_t\in\mathbb{R}^{1\times d}$, the objective function is basically a number of $d$ functions of $u_{t,j}\in\mathbb{R}$. For example, \autoref{eq:subproblem_u_swmtl_3} can be further split into the following separate problems for the decision variable $u_{t,j}$ which is a scalar: \\
For $j=1,2,3,\dots,d$:

\begin{equation}
\begin{split}
u_{1,j}^{k+1}\leftarrow \,& \arg\min_{u_{1,j}} c \max {(0,-u_{1,j}u_{2,j}^k)} - y_{1,j}^{k}u_{1,j} + (\rho/2)( w_{1,j}^{k+1}-u_{1,j})^{2}
\end{split}
\label{eq:sub_subproblem_u_swmtl_3}
\end{equation}

We can get the analytical solution for it as follows:

The RHS above is basically a piece-wise quadratic function depending on the sign of $u_{2,j}^k$. If $u_{2,j}^k < 0$, the analytical candidate solution to \autoref{eq:sub_subproblem_u_swmtl_3} is $[w_{1,j}^{k+1}+y_{1,j}^k/\rho]_{-}$ and $[w_{1,j}^{k+1}+(cu_{2,j}^k+y_{1,j}^k)/\rho]_{+}$. We only need to compare these two candidates by taking them back to \autoref{eq:sub_subproblem_u_swmtl_3} to see for which one the objective function is smaller. Similar case if $u_{2,j}^k \geq 0$. The analytical solutions for $u_{t,j}$ where $t>1$ can be obtained following similar way. 

Finally, we update the dual variable $y$ as follow:

\begin{equation}
\begin{split}
y^{k+1} = y^k + \rho([w_1^{k+1};\cdots;w_T^{k+1}]-[u_1^{k+1};\cdots;u_T^{k+1}])
\end{split}
\label{eq:subproblem_y_swmtl_3}
\end{equation}

%% file: appendix_theoretical.tex
\section{Theoretical Analysis}

\subsection{Theorem 1 Proof}
\label{error bound proof}
In this section, we provide the comprehensive proof for \autoref{thm:error bound}. First, we will give some necessary definition and lemma, and at the end of this section we present the proof for \autoref{thm:error bound}.

\begin{definition}
\label{def:F variable}
For a multisample $X \in (\mathbb{R}^d)^{mT}$, define the random variable
\begin{equation}
\begin{split}
F(\sigma)=F_{\sigma}\coloneqq \sup_{w\in\mathcal{F}_{\alpha,\beta}}\displaystyle\sum_{t,i}^{T,m}\sigma_{ti}\langle w_t,x_{ti}\rangle
\end{split}
\label{eq:r.v.def}
\end{equation}
where $\sigma$ stands for Rademacher variable, which is a set of i.i.d. uniform random variables on $\{-1,1\}$.
\end{definition}

\begin{lemma}
\label{lemma:F bound}
For the random variable $F_{\sigma}$, we have
\begin{equation}
\begin{split}
\mathbb{E}\{F_{\sigma}\}\leq \alpha\cdot\max_{1\leq t \leq T}\left\lVert x_t \right\rVert_{1,\infty}
\end{split}
\label{eq:lemma1}
\end{equation}
where the expectation is w.r.t. $\sigma$ and $x_t\in\mathbb{R}^{m\times d}$ is the input feature for the t-th task.
\end{lemma}

\begin{proof}
\begin{equation}
\begin{split}
\mathbb{E}\{F_{\sigma}\} & =  \mathbb{E}\{\sup_{w\in\mathcal{F}_{\alpha,\beta}}\displaystyle\sum_{t=1}^T\displaystyle\sum_{i=1}^m\sigma_{ti}\langle w_t, x_{ti}\rangle \} \ \ \text{(Definition~\autoref{def:F variable})}\\
& = \mathbb{E}\{\sup_{w\in\mathcal{F}_{\alpha,\beta}}\displaystyle\sum_{t=1}^T\langle w_t, \displaystyle\sum_{i=1}^m \sigma_{ti}x_{ti}\rangle \} \\
& = \mathbb{E}\{\sup_{w\in\mathcal{F}_{\alpha,\beta}}\displaystyle\sum_{t=1}^T\displaystyle\sum_{j=1}^d [(\displaystyle\sum_{i=1}^m x_{tij}\sigma_{ti})\cdot w_{tj}]\} \\
& = \alpha\cdot\mathbb{E}\{\max_{1\leq t \leq T, 1\leq j \leq d} |\displaystyle\sum_{i=1}^m x_{tij}\sigma_{ti}| \} \\
& \leq \alpha\cdot\max_{1\leq t \leq T}\left\lVert x_t \right\rVert_{1,\infty}
\end{split}
\label{eq:lemma1proof}
\end{equation}
The second equation is based on the linearity of inner product on $\mathbb{R}^d$. The third equation is simply reformulating the term inside the expectation to be a linear combination of each $w_{tj}$. The fourth equation is because for any given $X$ and $\sigma$, the term inside the $\sup$ is a linear function of $w_{tj}$. Meanwhile, recall the definition of $\mathcal{F}_{\alpha,\beta}$, the linear function of $w_{tj}$ will achieve the maximal value at the boundary of $\mathcal{F}_{\alpha,\beta}$. In fact, if we ignore the constraint of $\beta$, i.e. consider $\mathcal{F}_{\alpha} = \{w\in\mathbb{R}^{d \times T}: \textstyle\sum_{t=1}^T\left\lVert w_t \right\rVert_{1} \leq \alpha \}$, the linear function of $w_{tj}$ will achieve the maximal value at the vertexes of the region, where only one $w_{tj}$ equals to $\alpha$ while others equal to 0 (suppose the maxima is unique). Notice the constraint of $\beta$ only controls the sign of $w$ and doesn't have effect when $w_{tj}$ equals to 0, which corresponds to the case of vertexes. In other words, adding the constraint of $\beta$ back to our problem doesn't affect the maxima. Hence, the fourth equation holds. The last inequality is simply taking the possible maximal value inside the expectation, which is the $L_1$ norm of the column which has the largest $L_1$ norm among all the columns in different $x_t$ $(t=1,2,\cdots,T)$.
\end{proof}

\begin{theorem}
\label{thm:6}
$\forall{\epsilon} > 0$, let $\mu_1,\mu_2,\dots,\mu_T$ be the probability measure on $\mathbb{R}^d \times\mathbb{R}$. With probability of at least $1-\epsilon$ in the draw of $Z=(X,Y)\sim\prod_{t=1}^{T}{\mu_{t}^{m}}$, for any $w\in \mathcal{F}_{\alpha,\beta}$ we have:
\begin{equation}
\begin{split}
\mathbb{E}(w)-\mathbb{\hat{E}}(w|Z) & = \frac{1}{T} \displaystyle\sum_{t=1}^T\mathbb{E}_{(x,y)\sim{\mu_t}}[\mathcal{L}(\langle w_{t},x\rangle,y)] \\
&- \frac{1}{mT}\displaystyle\sum_{t,i}^{T,m}\mathcal{L}(\langle w_{t},x_{ti}\rangle,y_{ti}) \\
& \leq \frac{2L\alpha}{mT}\max_{1\leq t \leq T}\left\lVert x_t \right\rVert_{1,\infty} + \sqrt{\frac{9\ln{2/\epsilon}}{2mT}}
\end{split}
\label{eq:Thm3}
\end{equation}
\end{theorem}

\begin{proof}
By the Corollary 6 from~\cite{maurer2013sparse}, $\forall{\epsilon}>0$, we have with probability greater than $1-\epsilon$ that:
\begin{equation}
\begin{split}
\mathbb{E}(w) & = \frac{1}{T} \displaystyle\sum_{t=1}^T\mathbb{E}_{(x,y)\sim{\mu_t}}[\mathcal{L}(\langle w_{t},x\rangle,y)] \\
& \leq \frac{1}{mT}\displaystyle\sum_{t,i}^{T,m}\mathcal{L}(\langle w_{t},x_{ti}\rangle,y_{ti}) + \mathbb{E}\{\sup_{w\in\mathcal{F}_{\alpha,\beta}}\frac{2}{mT}\displaystyle\sum_{t=1}^T\displaystyle\sum_{i=1}^m\sigma_{ti}\mathcal{L}(\langle w_t, x_{ti}\rangle) \} + \sqrt{\frac{9\ln{2/\epsilon}}{2mT}} \\
& = \mathbb{\hat{E}}(w|Z) + \hat{\mathcal{R}} + \sqrt{\frac{9\ln{2/\epsilon}}{2mT}}
\end{split}
\label{eq:Thm3proof1}
\end{equation}

Here we simply denote the second term in the second last equation as $\hat{\mathcal{R}}$.

By Assumption~\ref{ass:obj_function}, Lemma~\ref{lemma:F bound} and the Lipschitz property from Lemma 7 in~\cite{maurer2013sparse}, we have:
\begin{equation}
\begin{split}
\frac{mT}{2}\hat{\mathcal{R}} & = \mathbb{E}\{\sup_{w\in\mathcal{F}_{\alpha,\beta}}\displaystyle\sum_{t=1}^T\displaystyle\sum_{i=1}^m\sigma_{ti}\mathcal{L}(\langle w_t, x_{ti}\rangle, y_{ti}) \}  \\
& \leq L\mathbb{E}\{\sup_{w\in\mathcal{F}_{\alpha,\beta}}\displaystyle\sum_{t=1}^T\displaystyle\sum_{i=1}^m\sigma_{ti}\langle w_t, x_{ti}\rangle \} \ \ \ \ \text{(Lemma 7 in ~\cite{maurer2013sparse})}\\
& \leq L\alpha\max_{1\leq t \leq T}\left\lVert x_t \right\rVert_{1,\infty} \ \ \ \ \text{(Lemma~\ref{lemma:F bound})}
\end{split}
\label{eq:Thm3proof2}
\end{equation}
By combining two results above gives the whole proof.
\end{proof}

Now we can give the proof for~\autoref{thm:error bound}.

\begin{proof}
Since $\forall{Z}$, $\mathbb{\hat{E}}(w^*|Z)-\mathbb{\hat{E}}(w_{(Z)}^*|Z)\geq 0$, we have:
\begin{equation}
\begin{split}
\mathbb{E}(w_{Z}^*) & = \mathbb{E}(w_{Z}^*) - \mathbb{E}(w^*) + \mathbb{E}(w^*) \leq \mathbb{\hat{E}}(w^*|Z)-\mathbb{\hat{E}}(w_{(Z)}^*|Z) + \mathbb{E}(w_{Z}^*) - \mathbb{E}(w^*) + \mathbb{E}(w^*)
\end{split}
\label{eq:Thm1proof1}
\end{equation}
Therefore,
\begin{equation}
\begin{split}
\mathbb{E}(w_{Z}^*) - \mathbb{E}(w^*) & \leq \mathbb{\hat{E}}(w^*|Z)-\mathbb{\hat{E}}(w_{(Z)}^*|Z) + \mathbb{E}(w_{Z}^*) - \mathbb{E}(w^*) \\ &\leq\sup_{w\in\mathcal{F}_{\alpha,\beta}}\{|\mathbb{E}(w)-\mathbb{\hat{E}}(w|Z)| \} + \mathbb{\hat{E}}(w^*|Z) - \mathbb{E}(w^*)
\end{split}
\label{eq:Thm1proof2}
\end{equation}
By the Hoeffding's inequality from Theorem 6.14 in~\cite{zhou2013feafiner}, we can bound the last two terms. Hence, by \autoref{thm:6}, with probability at least $1-\epsilon$, we have:
\begin{equation}
\begin{split}
\mathbb{E}(w_{Z}^*) - \mathbb{E}(w^*) & \leq \sup_{w\in\mathcal{F}_{\alpha,\beta}}|\mathbb{E}(w)-\mathbb{\hat{E}}(w|Z)| + \sqrt{\frac{\ln(2/\epsilon)}{2mT}}\\
& \leq \frac{2L\alpha}{mT}\max_{1\leq t \leq T}\left\lVert x_t \right\rVert_{1,\infty} + 2\sqrt{\frac{2\ln{2/\epsilon}}{mT}} 
\end{split}
\label{eq:Thm1proof3}
\end{equation}
Here the first inequality uses the Hoeffding's inequality while the second one uses \autoref{thm:6}.
\end{proof}

\subsection{Bound Comparison}
\label{bound comparison}
In this section, we provide another proof for calculating the upper bound for $\mathbb{E}\{F_{\sigma}\}$, by extending the Lemma 11 ~\cite{maurer2013sparse}, which is a very commonly used way for deriving the generalization error in multi-task learning problem. We will show that under some mild assumptions the error bound from Lemma~\ref{lemma:F bound} in our main paper is better (tighter) than that derived by Lemma 11 ~\cite{maurer2013sparse}.

\begin{lemma}
\label{lemma:F bound old}
For the random variable $F_{\sigma}$, we have
\begin{equation}
\begin{split}
\mathbb{E}\{F_{\sigma}\}\leq \alpha\sqrt{\displaystyle\sum_{t=1}^T\displaystyle\sum_{i=1}^m\left\lVert x_{ti}\right\rVert_{2}^{2}}
\end{split}
\label{eq:lemma2}
\end{equation}
where the expectation is w.r.t. $\sigma$ and $x_{ti}\in\mathbb{R}^{d}$ is the i-th instance of the input feature for the t-th task.
\end{lemma}

\begin{proof}
\begin{equation}
\begin{split}
\mathbb{E}\{F_{\sigma}\} & =  \mathbb{E}\{\sup_{w\in\mathcal{F}_{\alpha,\beta}}\displaystyle\sum_{t=1}^T\displaystyle\sum_{i=1}^m\sigma_{ti}\langle w_t, x_{ti}\rangle \} \\
& = \mathbb{E}\{\sup_{w\in\mathcal{F}_{\alpha,\beta}}\displaystyle\sum_{t=1}^T\langle w_t, \displaystyle\sum_{i=1}^m \sigma_{ti}x_{ti}\rangle \} \ \ \ \ \text{(Cauchy-Schwarz inequality)} \\
& \leq \mathbb{E}\{\sup_{w\in\mathcal{F}_{\alpha,\beta}}\displaystyle\sum_{t=1}^T\left\lVert w_t\right\rVert_{2}\cdot \left\lVert \displaystyle\sum_{i=1}^m \sigma_{ti}x_{ti}\right\rVert_{2} \} \ \ \ \  \text{(Cauchy-Schwarz inequality)} \\
& \leq \mathbb{E}\{\sup_{w\in\mathcal{F}_{\alpha,\beta}}\sqrt{\displaystyle\sum_{t=1}^T\left\lVert w_t\right\rVert_{2}^{2}}\cdot \sqrt{\displaystyle\sum_{t=1}^T\left\lVert \displaystyle\sum_{i=1}^m \sigma_{ti}x_{ti}\right\rVert_{2}^{2}} \} \\
& = \sup_{w\in\mathcal{F}_{\alpha,\beta}}\sqrt{\displaystyle\sum_{t=1}^T\left\lVert w_t\right\rVert_{2}^{2}} \cdot \mathbb{E}\{\sqrt{\displaystyle\sum_{t=1}^T\left\lVert \displaystyle\sum_{i=1}^m \sigma_{ti}x_{ti}\right\rVert_{2}^{2}}\} \\
& \leq \alpha\sqrt{\displaystyle\sum_{t=1}^T\displaystyle\sum_{i=1}^m\left\lVert x_{ti}\right\rVert_{2}^{2}}
\end{split}
\label{eq:lemma2proof_1}
\end{equation}
 The last inequality is because for any $w\in\mathcal{F}_{\alpha,\beta}$, we have:
\begin{equation}
\begin{split}
\sqrt{\displaystyle\sum_{t=1}^T\left\lVert w_t\right\rVert_{2}^{2}} \leq \displaystyle\sum_{t=1}^T\left\lVert w_t\right\rVert_{1} \leq \alpha
\end{split}
\label{eq:lemma2proof_2}
\end{equation}
In addition, by Jensen's Inequality, for any non-negative random variable $X$, $\mathbb{E}\{\sqrt{X}\} \leq \sqrt{\mathbb{E}\{X\}}$, and for Rademacher variable $\sigma$, $\mathbb{E}\{\sigma\}=0$ and $Var(\sigma)=1$. Hence, we have:
\begin{equation}
\begin{split}
& \mathbb{E}\{\sqrt{\displaystyle\sum_{t=1}^T\left\lVert \displaystyle\sum_{i=1}^m \sigma_{ti}x_{ti}\right\rVert_{2}^{2}}\} \leq \sqrt{\displaystyle\sum_{t=1}^T\mathbb{E}\{\left\lVert \displaystyle\sum_{i=1}^m \sigma_{ti}x_{ti}\right\rVert_{2}^{2}\}}\\ 
& = \sqrt{\displaystyle\sum_{t=1}^T\displaystyle\sum_{i=1}^m\left\lVert x_{ti}\right\rVert_{2}^{2}}
\end{split}
\label{eq:lemma2proof_3}
\end{equation}
\end{proof}

Recall the result in Lemma~\ref{lemma:F bound}, where
\begin{equation}
\begin{split}
\mathbb{E}\{F_{\sigma}\}\leq \alpha\cdot\max_{1\leq t \leq T}\left\lVert x_t \right\rVert_{1,\infty}
\end{split}
\label{eq:Lemma2_1}
\end{equation}

This bound is tighter than that in Lemma~\ref{lemma:F bound old} under the following mild assumption over $X$:

\begin{assumption}
\label{ass:tight condition}
For any input data $X\in(\mathbb{R}^d)^{mT}$, denote $t^*,j^*=\arg\max_{1\leq t\leq T,1\leq j\leq d}\displaystyle\sum_{i=1}^m |x_{tij}|$. The following inequality holds:
\begin{equation}
\begin{split}
\displaystyle\sum_{(t,j)\neq (t^*,j^*)} x_{tij}^2 > 2\displaystyle\sum_{1\leq k<l\leq m} |x_{t^*kj^*} x_{t^*lj^*}|
\end{split}
\label{eq:Assumption1}
\end{equation}
\end{assumption}

Now we prove that for any $X$ that has unique maxima $t^*,j^*=\arg\max_{t,j}\displaystyle\sum_{i=1}^m |x_{tij}|$, the bound in Lemma~\ref{lemma:F bound} is better than that in Lemma~\ref{lemma:F bound old} if and only if the above assumption holds.

\begin{lemma}
For any input data $X\in(\mathbb{R}^d)^{mT}$ with the uniqueness of $t^*,j^*$ defined above, 
\begin{equation}
\begin{split}
\alpha\cdot\max_{1\leq t \leq T}\left\lVert x_t \right\rVert_{1,\infty} < \alpha\sqrt{\displaystyle\sum_{t=1}^T\displaystyle\sum_{i=1}^m\left\lVert x_{ti}\right\rVert_{2}^{2}}
\end{split}
\label{eq:lemma41}
\end{equation}
if and only if Assumption~\ref{ass:tight condition} holds.
\end{lemma}

\begin{proof} Notice for any positive $\alpha$, it can be cancelled without any effect on the proof.
\begin{equation}
\begin{split}
& \ \ \ \ \displaystyle\sum_{t=1}^T\displaystyle\sum_{i=1}^m\left\lVert x_{ti}\right\rVert_{2}^{2} - (\max_{1\leq t \leq  T}\left\lVert x_t \right\rVert_{1,\infty})^2 \\
& = \displaystyle\sum_{t=1}^T\displaystyle\sum_{i=1}^m\left\lVert x_{ti}\right\rVert_{2}^{2} - (\max_{1\leq t \leq T, 1\leq j \leq d} \displaystyle\sum_{i=1}^m |x_{tij}|)^2 \\
& =  \displaystyle\sum_{t=1}^T\displaystyle\sum_{i=1}^m\left\lVert x_{ti}\right\rVert_{2}^{2} - (\left\lVert x_{t*j*} \right\rVert_{1})^2 \\
& = \displaystyle\sum_{t,i,j}^{T,m,d} x_{tij}^2 - \{\displaystyle\sum_{i=1}^m x_{t*ij*}^2 + 2\displaystyle\sum_{1\leq k<l\leq m} |x_{t^*kj^*} x_{t^*lj^*}|\} \\
& = \displaystyle\sum_{(t,j)\neq (t^*,j^*)} x_{tij}^2 - 2\displaystyle\sum_{1\leq k<l\leq m} |x_{t^*kj^*}x_{t^*lj^*}| 
\end{split}
\label{eq:lemma4proof}
\end{equation}
Hence, the bound in Lemma~\ref{lemma:F bound} is tighter than that in Lemma~\ref{lemma:F bound old} is equivalent to the term on RHS of the last equation in \autoref{eq:lemma4proof} is negative, i.e. the Assumption~\ref{ass:tight condition} holds.
\end{proof}

\subsection{Convergence Analysis}
\label{appendix:convergence_analysis}

The proof for \autoref{thm:convergence} is as follow:

\begin{proof}
\label{appendix:global_convergence_proof}
The SRML model with $L_1$ regularization takes the form:
\begin{equation}
\begin{split}
 \min_{\substack{w_1,\cdots,w_T \\ u_1,\cdots,u_T}} \displaystyle\sum_{t=1}^T&\mathcal{L}_t(w_t) + \lambda\displaystyle\sum_{t=1}^T \left\lVert u_t \right\rVert_{1} + c\cdot\displaystyle\sum_{t=1}^{T-1}\displaystyle\sum_{j=1}^d \max {(0,-u_{t,j}u_{t+1,j})} \\
\text{s.t.\quad} &  {[w_1;\cdots;w_T]-[u_1;\cdots;u_T]=0}
\end{split}
\label{eq:swmtl_new}
\end{equation}
where $\mathcal{L}$ is either least square loss or logistic loss. The above problem amounts to a non-convex objective with equality constraint one, which is a special case of the following multi-convex inequality-constrained problem:
\begin{equation}
\begin{split}
\min \nolimits_{x_1,\cdots,x_n,z}& F(x_1,\cdots,x_n,z) = f(x_1,\cdots,x_n) + \sum\nolimits_{i=1}^n g_i(x_i)+h(z) \\
& \text{s.t.} l(x_1,\cdots,x_n)\leq 0, \ \ \ \ \sum\nolimits_{i=1}^n A_i x_i-z=0
\end{split}
\label{eq:miADMM}
\end{equation}
For this type of problem, ~\cite{wang2019multi} provided the sufficient conditions for proving the global convergence when using multi-convex inequality-constrained Alternating Direction Method of Multipliers (miADMM), which amounts to the following: \\
(1) (Regularity of $f$ and $l$) $f(x_1,\cdots,x_n)$ and $l(x_1,\cdots,x_n)$ are proper, continuous, multi-convex and possibly non-smooth functions. \\
(2) (Regularity of $g_i$) $g_i(x_i)$ $(i=1,\cdots,n)$ are proper, continuous, convex and possibly non-smooth functions. \\
(3) (Regularity of $h$) $h(z)$ is a proper, convex and Lipschitz differentiable (with constant $H$) function. 

Since SRML does not have the inequality constraint, the regularity of $l$ is satisfied. To fit SRML into miADMM, we treat our loss term as $h(z)$, our regularization term as $g_i$ $(i=1,\cdots,n)$, and slacking term as $f(x_1,\cdots,x_n)$. Now it suffices to prove each term of SRML satisfies the above conditions. 

First, we prove the regularity of $f$, i.e. the first condition, which corresponds to our slacking term. Since the slacking term as a function of any $u_{t,j}$ with other $u$ fixed is basically $y = c\cdot max(0,ax)$, where $a$ is a constant, it's simply proper, continuous and convex. In addition, this function is non-smooth only at $x=0$. Hence, the slacking term is proper, continuous, multi-convex and non-smooth and can be fit into $f(x_1,\cdots,x_n)$.

Second, we prove the regularity of $g_i$ $(i=1,\cdots,n)$, i.e. the second condition, which corresponds to our $L_1$ regularization term. Since the $L_1$ norm of $w_t$ is simply a sum of absolute value function, it's proper, continuous and non-smooth at $w = 0$. The multi-convex is also trivial to prove since $L_1$ norm is a separate function of each weight.

Third, we prove the regularity of $h(z)$, i.e. the last condition, which corresponds to our loss function. In our paper, we consider both regression and classification problem, so the loss function will be either least square loss or logistic loss. Next, We will prove in both case the loss function satisfies the condition. 

For least square loss, $\mathcal{L}_t(w_t) = \left\lVert Y_t - X_t w_t \right\rVert_{2}^2$, which is a quadratic function w.r.t. $w_t$. Hence, it's easy to show the function is proper and convex. Now we want to show it's also Lipschitz differentiable. Since the loss for different tasks are seperate, it suffices to show for one single task the loss function $\mathcal{L}_t(w_t)$ is Lipschitz differentiable. \\ \\ \\ \\ \\

For any $w_t^{'}$, $w_t^{''}$ $\in \mathbb{R}^d$,
\begin{equation}
\begin{split}
\left\lVert \nabla\mathcal{L}_t(w_t^{'}) - \nabla\mathcal{L}_t(w_t^{''}) \right\rVert & = \left\lVert (2X_t^{\intercal}X_t w_t^{'}-2X_t^{\intercal}Y_t) - (2X_t^{\intercal}X_t w_t^{''}-2X_t^{\intercal}Y_t) \right\rVert \\
& = \left\lVert 2X_t^{\intercal}X_t(w_t^{'} - w_t^{''}) \right\rVert \\
& \leq 2\left\lVert X_t^{\intercal}X_t\right\rVert \cdot \left\lVert w_t^{'} - w_t^{''} \right\rVert
\end{split}
\label{eq:ls_loss-Lipschitz}
\end{equation}

For the logistic loss, it's defined as follow:
\begin{equation}
\begin{split}
\mathcal{L}_t(w_t) = \frac{1}{m}\displaystyle\sum_{j=1}^{m}[-Y_{tj}\log\sigma(X_{tj}w_t) - (1-Y_{tj})\log\sigma(-X_{tj}w_t)] 
\end{split}
\label{eq:logisticloss}
\end{equation}
where $\sigma()$ is the sigmoid function.

For any $w_t^{'}$, $w_t^{''}$ $\in \mathbb{R}^d$,
\begin{equation}
\small
\begin{split}
&  \ \ \ \ \left\lVert \nabla\mathcal{L}_t(w_t^{'}) - \nabla\mathcal{L}_t(w_t^{''}) \right\rVert \\
& = \left\lVert \frac{1}{m}\displaystyle\sum_{j=1}^{m}\{[Y_{tj}-\sigma(X_{tj}w_t^{'})]\cdot X_{tj} - [Y_{tj}-\sigma(X_{tj}w_t^{''})]\cdot X_{tj} \} \right\rVert \\
& \leq \frac{1}{m}\displaystyle\sum_{j=1}^{m} \left\lVert X_{tj} \right\rVert \cdot\left\lVert [Y_{tj}-\sigma(X_{tj}w_t^{'})] - [Y_{tj}-\sigma(X_{tj}w_t^{''})]\right\rVert \\
& = \frac{1}{m}\displaystyle\sum_{j=1}^{m}\left\lVert X_{tj} \right\rVert \cdot \left\lVert \sigma(X_{tj}w_t^{''}) - \sigma(X_{tj}w_t^{'}) \right\rVert \\
& = \frac{1}{m}\displaystyle\sum_{j=1}^{m}\left\lVert X_{tj} \right\rVert \cdot \left\lVert \sigma^{'}(\xi_j) \cdot (X_{tj}w_t^{''} - X_{tj}w_t^{'}) \right\rVert \\
& \leq \frac{1}{m}\displaystyle\sum_{j=1}^{m} \left\lVert X_{tj} \right\rVert^{2} \cdot | \sigma^{'}(\xi_j) | \cdot \left\lVert w_t^{'} - w_t^{''} \right\rVert \\
& \leq \frac{1}{m}\displaystyle\sum_{j=1}^{m} \left\lVert X_{tj} \right\rVert^{2} \cdot \left\lVert w_t^{'} - w_t^{''} \right\rVert \ \ \text{($\sigma^{'}(\cdot) < 1$)}
\end{split}
\label{eq:logistic_loss-Lipschitz}
\end{equation}
Here the fourth equality holds by using the mean value theorem, where $\xi_j \in (X_{tj}w_t^{'},X_{tj}w_t^{''})$ (suppose $X_{tj}w_t^{'}\leq X_{tj}w_t^{''}$). The condition for mean value theorem is the function is continuous on the closed interval and differentiable on the open interval, which is satisfied by the logistic function.The last inequality is because the derivative of the logistic function is upper bounded by 1, i.e. $\sigma^{'}(x) < 1, \ \ \forall{x\in \mathbb{R}}$.
\end{proof}

The proof for \autoref{thm:Converge Nash point} is as follow:

\begin{proof}
We first prove all the $A_i$ in \autoref{eq:miADMM} are of full rank, and then prove the convergence to a Nash point. \\
Consider the SRML problem \autoref{eq:swmtl_new}, the equality constraint is $[w_1;\cdots;w_T]-[u_1;\cdots;u_T]=0$. To fit this into the notation of \autoref{eq:miADMM}, we only need to find a sequence of $\{A_t\}, \ \  t\in \{1,2,\cdots,T\}$, s.t.
\begin{equation}
\begin{split}
\displaystyle\sum_{t=1}^{T}A_t w_t = [w_1;\cdots;w_T]
\end{split}
\label{eq:A_t}
\end{equation}
We can define $A_t=[O_1;\cdots;O_{t-1};I_t;O_{t+1};\cdots;O_T]$, where each $O_{t}$ is a $d\times d$ zero matrix and $I_t$ is a $d\times d$ identity matrix. Notice each $A_t$ is a $(d\cdot T)\times d$ matrix and part of it is an identity matrix with same width, so $A_t$ is of full column rank. Since $A_t$ has more rows than columns, $A_t$ is of full rank. \\
Considering it is sufficient condition to prove the convergence to a Nash point as demonstrated in Theorem 2~\cite{wang2019multi}, the proof is completed.
\end{proof}

The proof for \autoref{thm:converge rate} is as follow:
\begin{proof}
The proof for \autoref{thm:converge rate} simply follows same procedure as Theorem 3~\cite{wang2019multi}, which is similar to Lemma 1.2 in~\cite{deng2017parallel}.
\end{proof}

%% file: appendix_experiments.tex
\section{Experiments}

\subsection{Real-world Datasets}
\label{sec:real_world_datasets}

This section describes the real-world datasets used to evaluate the performance of our approach and comparison methods.

\begin{description}

\item[School]
records the examination scores of 15362 students from 139 secondary schools during the three years from 1985 to 1987 in London~\cite{10.2307/2348496}\footnote{\url{http://ttic.uchicago.edu/~argyriou/code/index.html}}. 
Each student row contains 26 binary features, including school-specific and student-specific attributes. 
The corresponding examination score is an integer. 
The problem of predicting the examination score of the students formulated as a multi-task regression problem assign each school corresponds to a task. 

\item[Computer buyers survey] 
multi-output regression dataset obtained from a survey of 190 people about their likelihood of purchasing  20 different personal computers~\cite{lenk1996hierarchical}\footnote{\url{https://github.com/probml/pmtk3/tree/master/data/conjointAnalysisComputerBuyers}}. 
Each computer row contains 13 binary variables related to specifications. 
Moreover, each task has ratings (on a scale of 0 to 10) given by a person to each of the 20 computers. 

\item[Facebook metrics]
related to posts published during the year of 2014 on the Facebook's page of a renowned cosmetics brand.
The dataset contains 500  rows and part of the features analyzed by~\cite{moro2016predicting}\footnote{\url{https://archive.ics.uci.edu/ml/datasets/Facebook+metrics}}.
It includes seven features known before post-publication and 12 features for evaluating post-impact.
We use the category attribute as the task indicator, which yields three tasks.
The total number of interactions of each post is the target variable.
 
\item[Traffic SP] 
related to measurements of the behavior of the urban traffic of the city of Sao Paulo in Brazil~\cite{Ferreira2016CombinationOA}\footnote{\url{https://archive.ics.uci.edu/ml/datasets/Behavior+of+the+urban+traffic+of+the+city+of+Sao+Paulo+in+Brazil}}. 
There are 135 records from Monday to Friday during the week of December 14, 2009. 
The measurements include records from 7:00 to 20:00 every 30 minutes. We define two tasks: measurements before and after midday. 
The target variable is the percentage of slowness in the traffic.

\item[Cars]
this dataset contains the specifications for 428 new vehicles for the 2004 year~\cite{lock19931993}\footnote{\url{https://github.com/probml/pmtk3/tree/master/data/04cars}}. 
The variables include price, measurements about the size of the vehicle, and fuel efficiency.
Each task is related to a type of car specified by the first four binary variables, resulting in four tasks. 
The other variables are part of the features, and the price is the target variable for each task. 
This dataset contains mixed variable types, i.e., categorical and real values.



\end{description}

\subsection{Additional Results on Feature Learning}
\label{results on feature learning}



In this section, we perform analysis of  the learned feature weights (\autoref{fig:weightstasksclf}) for the proposed model and the comparison methods using both synthetic and real-world datasets.

For each model, \autoref{fig:weightstasksclf} shows the weights of the different features (different rows) for different tasks (different columns). For each cell, the color specifies the signs of the feature weights. Specifically, if a feature weight is zero then it is colored by white, if a feature weight is positive then \textit{red} is used while it is \textit{blue} when the weight value is negative.
The first and second rows show the learned weights for the Synthetic Datasets 1 and 3, respectively, while the remaining rows are for the real-world datasets:  School (3rd row), Facebook (4th row), and TrafficSP (5th row).
In the case of our model SRML, we can see that all features across different tasks tend to better maintain the same polarity; in contrast, the comparison models cannot maintain them well, which is even more obvious in the real-world datasets. Since synthetic datasets (i.e., the first two rows) has ground truth weights (i.e., the first column), we can see that the feature weight signs learned by our model is very close to the ground truth, while the comparison methods typically perform poorly.
This study demonstrates that our method can perform better when the features across tasks share similar polarity of the weights.

\begin{figure*}[h]
\begin{center}

\begin{subfigure}[b]{\textwidth}
\includegraphics[width=\textwidth]{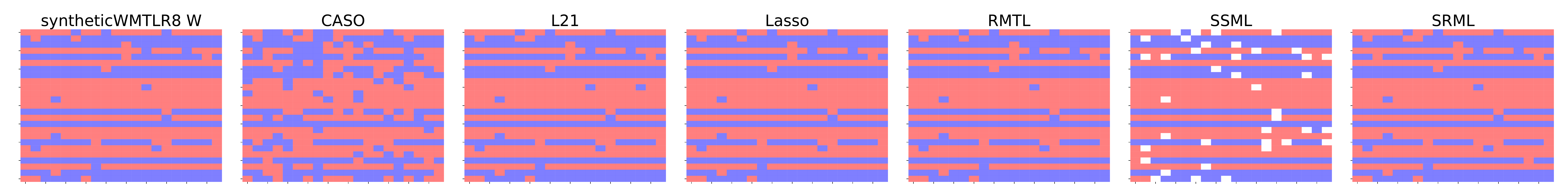}
\label{default}
\end{subfigure}

\begin{subfigure}[b]{\textwidth}
\includegraphics[width=\textwidth]{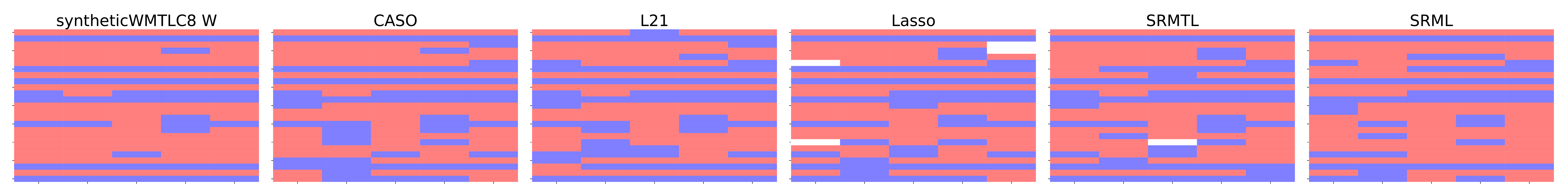}
\label{default}
\end{subfigure}

\begin{subfigure}[b]{\textwidth}
\includegraphics[width=\textwidth]{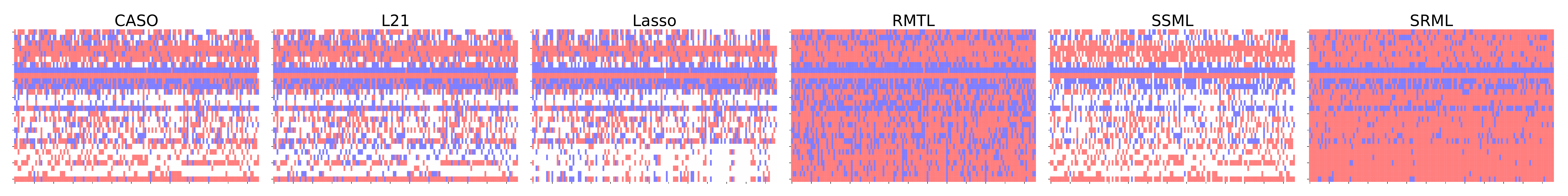}
\label{default}
\end{subfigure}

\begin{subfigure}[b]{\textwidth}
\includegraphics[width=\textwidth]{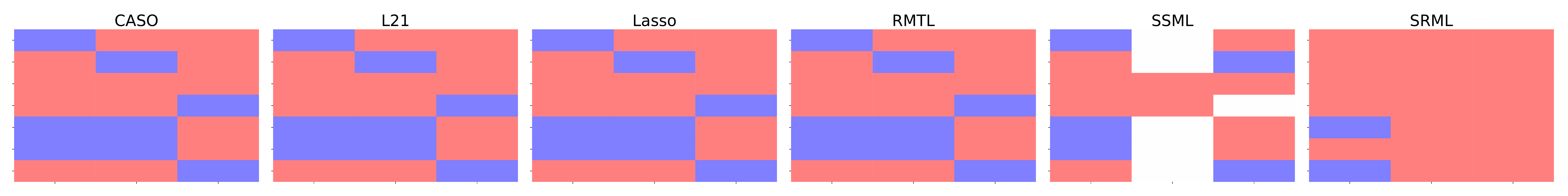}
\label{default}
\end{subfigure}

\begin{subfigure}[b]{\textwidth}
\includegraphics[width=\textwidth]{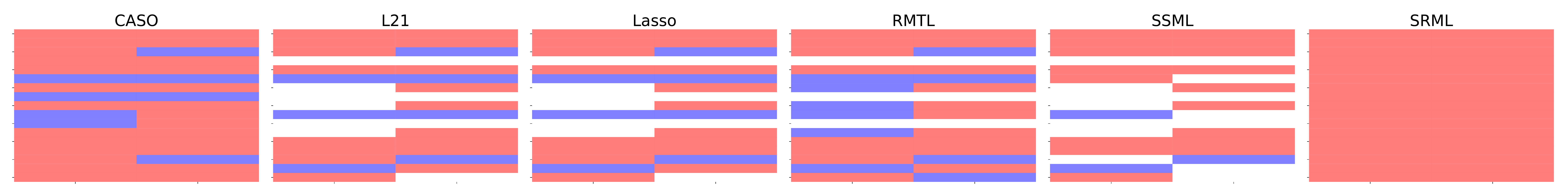}
\label{default}
\end{subfigure}

\caption{Features selection for regression task on different datasets (each row per dataset).
For each model (subfigure), the weights ($y$ axis) of each task ($x$ axis) specify the polarity of selected features either positive (\textit{red}) or negative (\textit{blue}), and not selected features as (\textit{white}).
}
\label{fig:weightstasksclf}
\end{center}
\end{figure*}

%% file: main.bbl
\begin{thebibliography}{36}
\providecommand{\natexlab}[1]{#1}
\providecommand{\url}[1]{\texttt{#1}}
\expandafter\ifx\csname urlstyle\endcsname\relax
  \providecommand{\doi}[1]{doi: #1}\else
  \providecommand{\doi}{doi: \begingroup \urlstyle{rm}\Url}\fi

\bibitem[Argyriou et~al.(2007)Argyriou, Evgeniou, and
  Pontil]{argyriou2007multi}
Andreas Argyriou, Theodoros Evgeniou, and Massimiliano Pontil.
\newblock Multi-task feature learning.
\newblock In \emph{Advances in neural information processing systems}, pages
  41--48, 2007.

\bibitem[Argyriou et~al.(2008)Argyriou, Evgeniou, and
  Pontil]{argyriou2008convex}
Andreas Argyriou, Theodoros Evgeniou, and Massimiliano Pontil.
\newblock Convex multi-task feature learning.
\newblock \emph{Machine learning}, 73\penalty0 (3):\penalty0 243--272, 2008.

\bibitem[Bishop(2006)]{bishop2006pattern}
Christopher~M Bishop.
\newblock \emph{Pattern recognition and machine learning}.
\newblock springer, 2006.

\bibitem[Boyd et~al.(2011)Boyd, Parikh, Chu, Peleato, Eckstein,
  et~al.]{boyd2011distributed}
Stephen Boyd, Neal Parikh, Eric Chu, Borja Peleato, Jonathan Eckstein, et~al.
\newblock Distributed optimization and statistical learning via the alternating
  direction method of multipliers.
\newblock \emph{Foundations and Trends{\textregistered} in Machine learning},
  3\penalty0 (1):\penalty0 1--122, 2011.

\bibitem[Chen et~al.(2009)Chen, Tang, Liu, and Ye]{chen2009convex}
Jianhui Chen, Lei Tang, Jun Liu, and Jieping Ye.
\newblock A convex formulation for learning shared structures from multiple
  tasks.
\newblock In \emph{Proceedings of the 26th Annual International Conference on
  Machine Learning}, pages 137--144. ACM, 2009.

\bibitem[Chen et~al.(2011)Chen, Zhou, and Ye]{chen2011integrating}
Jianhui Chen, Jiayu Zhou, and Jieping Ye.
\newblock Integrating low-rank and group-sparse structures for robust
  multi-task learning.
\newblock In \emph{Proceedings of the 17th ACM SIGKDD international conference
  on Knowledge discovery and data mining}, pages 42--50. ACM, 2011.

\bibitem[Deng et~al.(2017)Deng, Lai, Peng, and Yin]{deng2017parallel}
Wei Deng, Ming-Jun Lai, Zhimin Peng, and Wotao Yin.
\newblock Parallel multi-block admm with o (1/k) convergence.
\newblock \emph{Journal of Scientific Computing}, 71\penalty0 (2):\penalty0
  712--736, 2017.

\bibitem[Evgeniou and Pontil(2004)]{evgeniou2004regularized}
Theodoros Evgeniou and Massimiliano Pontil.
\newblock Regularized multi--task learning.
\newblock In \emph{Proceedings of the tenth ACM SIGKDD international conference
  on Knowledge discovery and data mining}, pages 109--117, 2004.

\bibitem[Ferreira(2016)]{Ferreira2016CombinationOA}
Ricardo~Pinto Ferreira.
\newblock Combination of artificial intelligence techniques for prediction the
  behavior of urban vehicular traffic in the city of s{\~a}o paulo.
\newblock 2016.

\bibitem[Goldstein(1991)]{10.2307/2348496}
Harvey Goldstein.
\newblock Multilevel modelling of survey data.
\newblock \emph{Journal of the Royal Statistical Society. Series D (The
  Statistician)}, 40\penalty0 (2):\penalty0 235--244, 1991.
\newblock ISSN 00390526, 14679884.
\newblock URL \url{http://www.jstor.org/stable/2348496}.

\bibitem[G{\"o}rnitz et~al.(2011)G{\"o}rnitz, Widmer, Zeller, Kahles,
  R{\"a}tsch, and Sonnenburg]{gornitz2011hierarchical}
Nico G{\"o}rnitz, Christian Widmer, Georg Zeller, Andr{\'e} Kahles, Gunnar
  R{\"a}tsch, and S{\"o}ren Sonnenburg.
\newblock Hierarchical multitask structured output learning for large-scale
  sequence segmentation.
\newblock In \emph{Advances in Neural Information Processing Systems}, pages
  2690--2698, 2011.

\bibitem[Jacob et~al.(2009)Jacob, Vert, and Bach]{jacob2009clustered}
Laurent Jacob, Jean-philippe Vert, and Francis~R Bach.
\newblock Clustered multi-task learning: A convex formulation.
\newblock In \emph{Advances in neural information processing systems}, pages
  745--752, 2009.

\bibitem[Jalali et~al.(2010)Jalali, Sanghavi, Ruan, and
  Ravikumar]{jalali2010dirty}
Ali Jalali, Sujay Sanghavi, Chao Ruan, and Pradeep~K Ravikumar.
\newblock A dirty model for multi-task learning.
\newblock In \emph{Advances in neural information processing systems}, pages
  964--972, 2010.

\bibitem[Kumar and Daum{\'e}~III(2012)]{kumar2012learning}
Abhishek Kumar and Hal Daum{\'e}~III.
\newblock Learning task grouping and overlap in multi-task learning.
\newblock In \emph{Proceedings of the 29th International Coference on
  International Conference on Machine Learning}, pages 1723--1730, 2012.

\bibitem[Lenk et~al.(1996)Lenk, DeSarbo, Green, and
  Young]{lenk1996hierarchical}
Peter~J Lenk, Wayne~S DeSarbo, Paul~E Green, and Martin~R Young.
\newblock Hierarchical bayes conjoint analysis: Recovery of partworth
  heterogeneity from reduced experimental designs.
\newblock \emph{Marketing Science}, 15\penalty0 (2):\penalty0 173--191, 1996.

\bibitem[Liang et~al.(2019)Liang, Yang, Chen, Hu, and Urtasun]{liang2019multi}
Ming Liang, Bin Yang, Yun Chen, Rui Hu, and Raquel Urtasun.
\newblock Multi-task multi-sensor fusion for 3d object detection.
\newblock In \emph{Proceedings of the IEEE Conference on Computer Vision and
  Pattern Recognition}, pages 7345--7353, 2019.

\bibitem[Liu et~al.(2009)Liu, Ji, and Ye]{liu2009multi}
Jun Liu, Shuiwang Ji, and Jieping Ye.
\newblock Multi-task feature learning via efficient l2, 1-norm minimization.
\newblock In \emph{Proceedings of the Twenty-Fifth Conference on Uncertainty in
  Artificial Intelligence}, pages 339--348, 2009.

\bibitem[Liu et~al.(2019{\natexlab{a}})Liu, Johns, and Davison]{Liu_2019_CVPR}
Shikun Liu, Edward Johns, and Andrew~J. Davison.
\newblock End-to-end multi-task learning with attention.
\newblock In \emph{The IEEE Conference on Computer Vision and Pattern
  Recognition (CVPR)}, June 2019{\natexlab{a}}.

\bibitem[Liu and Pan(2017)]{liu2017adaptive}
Sulin Liu and Sinno~Jialin Pan.
\newblock Adaptive group sparse multi-task learning via trace lasso.
\newblock In \emph{IJCAI}, pages 2358--2364, 2017.

\bibitem[Liu et~al.(2019{\natexlab{b}})Liu, He, Chen, and
  Gao]{liu-etal-2019-multi}
Xiaodong Liu, Pengcheng He, Weizhu Chen, and Jianfeng Gao.
\newblock Multi-task deep neural networks for natural language understanding.
\newblock In \emph{Proceedings of the 57th Annual Meeting of the Association
  for Computational Linguistics}, pages 4487--4496, Florence, Italy, July
  2019{\natexlab{b}}. Association for Computational Linguistics.
\newblock \doi{10.18653/v1/P19-1441}.
\newblock URL \url{https://www.aclweb.org/anthology/P19-1441}.

\bibitem[Lock(1993)]{lock19931993}
Robin~H Lock.
\newblock 1993 new car data.
\newblock \emph{Journal of Statistics Education}, 1\penalty0 (1), 1993.

\bibitem[Maurer et~al.(2013)Maurer, Pontil, and
  Romera-Paredes]{maurer2013sparse}
Andreas Maurer, Massi Pontil, and Bernardino Romera-Paredes.
\newblock Sparse coding for multitask and transfer learning.
\newblock In \emph{International conference on machine learning}, pages
  343--351, 2013.

\bibitem[Moro et~al.(2016)Moro, Rita, and Vala]{moro2016predicting}
S{\'e}rgio Moro, Paulo Rita, and Bernardo Vala.
\newblock Predicting social media performance metrics and evaluation of the
  impact on brand building: A data mining approach.
\newblock \emph{Journal of Business Research}, 69\penalty0 (9):\penalty0
  3341--3351, 2016.

\bibitem[Pan and Yang(2009)]{pan2009survey}
Sinno~Jialin Pan and Qiang Yang.
\newblock A survey on transfer learning.
\newblock \emph{IEEE Transactions on knowledge and data engineering},
  22\penalty0 (10):\penalty0 1345--1359, 2009.

\bibitem[Tibshirani(1996)]{tibshirani1996regression}
Robert Tibshirani.
\newblock Regression shrinkage and selection via the lasso.
\newblock \emph{Journal of the Royal Statistical Society: Series B
  (Methodological)}, 58\penalty0 (1):\penalty0 267--288, 1996.

\bibitem[Wang and Zhao(2017)]{wang2017nonconvex}
Junxiang Wang and Liang Zhao.
\newblock Nonconvex generalization of admm for nonlinear equality constrained
  problems.
\newblock \emph{arXiv preprint arXiv:1705.03412}, 2017.

\bibitem[Wang et~al.(2018)Wang, Gao, Z{\"u}fle, Yang, and
  Zhao]{wang2018incomplete}
Junxiang Wang, Yuyang Gao, Andreas Z{\"u}fle, Jingyuan Yang, and Liang Zhao.
\newblock Incomplete label uncertainty estimation for petition victory
  prediction with dynamic features.
\newblock In \emph{2018 IEEE International Conference on Data Mining (ICDM)},
  pages 537--546. IEEE, 2018.

\bibitem[Wang et~al.(2019)Wang, Zhao, and Wu]{wang2019multi}
Junxiang Wang, Liang Zhao, and Lingfei Wu.
\newblock Multi-convex inequality-constrained alternating direction method of
  multipliers.
\newblock \emph{arXiv preprint arXiv:1902.10882}, 2019.

\bibitem[Wang et~al.(2016)Wang, Bi, Yu, Sun, and Song]{wang2016multiplicative}
Xin Wang, Jinbo Bi, Shipeng Yu, Jiangwen Sun, and Minghu Song.
\newblock Multiplicative multitask feature learning.
\newblock \emph{The Journal of Machine Learning Research}, 17\penalty0
  (1):\penalty0 2820--2852, 2016.

\bibitem[Yao et~al.(2019)Yao, Cao, and Chen]{yao2019robust}
Yaqiang Yao, Jie Cao, and Huanhuan Chen.
\newblock Robust task grouping with representative tasks for clustered
  multi-task learning.
\newblock In \emph{Proceedings of the 25th ACM SIGKDD International Conference
  on Knowledge Discovery \& Data Mining}, pages 1408--1417, 2019.

\bibitem[Zhang and Yang(2017)]{zhang2017survey}
Yu~Zhang and Qiang Yang.
\newblock A survey on multi-task learning.
\newblock \emph{arXiv preprint arXiv:1707.08114}, 2017.

\bibitem[Zhao et~al.(2019)Zhao, Hong, Wei, Chen, Nath, Andrews, Kumthekar,
  Sathiamoorthy, Yi, and Chi]{10.1145/3298689.3346997}
Zhe Zhao, Lichan Hong, Li~Wei, Jilin Chen, Aniruddh Nath, Shawn Andrews, Aditee
  Kumthekar, Maheswaran Sathiamoorthy, Xinyang Yi, and Ed~Chi.
\newblock Recommending what video to watch next: A multitask ranking system.
\newblock In \emph{Proceedings of the 13th ACM Conference on Recommender
  Systems}, RecSys '19, pages 43--51, New York, NY, USA, 2019. Association for
  Computing Machinery.
\newblock ISBN 9781450362436.
\newblock \doi{10.1145/3298689.3346997}.
\newblock URL \url{https://doi.org/10.1145/3298689.3346997}.

\bibitem[Zhou et~al.(2011)Zhou, Chen, and Ye]{zhou2011clustered}
Jiayu Zhou, Jianhui Chen, and Jieping Ye.
\newblock Clustered multi-task learning via alternating structure optimization.
\newblock In \emph{Advances in neural information processing systems}, pages
  702--710, 2011.

\bibitem[Zhou et~al.(2013{\natexlab{a}})Zhou, Lu, Sun, Yuan, Wang, and
  Ye]{zhou2013feafiner}
Jiayu Zhou, Zhaosong Lu, Jimeng Sun, Lei Yuan, Fei Wang, and Jieping Ye.
\newblock Feafiner: biomarker identification from medical data through feature
  generalization and selection.
\newblock In \emph{Proceedings of the 19th ACM SIGKDD international conference
  on Knowledge discovery and data mining}, pages 1034--1042,
  2013{\natexlab{a}}.

\bibitem[Zhou and Zhao(2015)]{zhou2015flexible}
Qiang Zhou and Qi~Zhao.
\newblock Flexible clustered multi-task learning by learning representative
  tasks.
\newblock \emph{IEEE transactions on pattern analysis and machine
  intelligence}, 38\penalty0 (2):\penalty0 266--278, 2015.

\bibitem[Zhou et~al.(2013{\natexlab{b}})Zhou, Wang, Jia, and
  Zhao]{zhou2013learning}
Qiang Zhou, Gang Wang, Kui Jia, and Qi~Zhao.
\newblock Learning to share latent tasks for action recognition.
\newblock In \emph{Proceedings of the IEEE International Conference on Computer
  Vision}, pages 2264--2271, 2013{\natexlab{b}}.

\end{thebibliography}
